\documentclass[11pt, a4paper, logo, copyright]{neutral}

\usepackage{natbib}
\usepackage{commands}
\usepackage{algorithm}
\usepackage{algpseudocode}
\usepackage{xltabular}
\usepackage{lineno}

\usepackage[utf8]{inputenc} 
\usepackage[T1]{fontenc}    
\usepackage{url}            
\usepackage{booktabs}       
\usepackage{amsfonts}       
\usepackage{nicefrac}       
\usepackage{microtype}      
\usepackage{xcolor}         

\usepackage{microtype}
\usepackage{graphicx}
\usepackage{booktabs}
\usepackage{amsmath,bm}
\usepackage{amsthm}
\usepackage{wrapfig}
\usepackage{color, soul}
\usepackage{psfrag}
\usepackage{adjustbox}
\usepackage{xspace}
\usepackage{amssymb}
\usepackage{multirow}
\usepackage{afterpage}
\usepackage{colortbl}
\usepackage{float}
\usepackage{textcomp}
\usepackage{siunitx}
\usepackage{mdframed}
\usepackage{makecell}
\usepackage[super]{nth}

\usepackage[inline]{enumitem}
\usepackage{bbding}
\usepackage{pifont}
\usepackage{bbm}
\usepackage{xcolor}
\usepackage{pdflscape}
\usepackage{theoremref}
\usepackage{pifont}
\usepackage{adjustbox}
\usepackage{xcolor}
\usepackage{mathtools}
\usepackage{subcaption}
\usepackage{pbox}
\usepackage{longtable}
\usepackage{listings}

\colorlet{darkgreen}{green!65!black}
\colorlet{darkblue}{blue!75!black}
\colorlet{darkred}{red!80!black}
\definecolor{statistical}{HTML}{8c564b}
\definecolor{structural}{HTML}{0070C0}
\definecolor{semantic}{HTML}{008080}
\definecolor{yellow}{HTML}{f7c600}
\definecolor{lightblue}{HTML}{0071bc}
\definecolor{lightgreen}{HTML}{39b54a}
\definecolor{deemph}{gray}{0.55}

\definecolor{baselinecolor}{gray}{.95}
\definecolor{graycolor}{gray}{.95}

\usepackage[pagebackref=false, breaklinks=true, colorlinks, citecolor=lightblue, urlcolor=lightblue, bookmarks=false]{hyperref}

\newlength\savewidth
\newcolumntype{x}[1]{>{\centering\arraybackslash}p{#1pt}}
\newcolumntype{y}[1]{>{\raggedright\arraybackslash}p{#1pt}}
\newcolumntype{z}[1]{>{\raggedleft\arraybackslash}p{#1pt}}

\definecolor{textgreen}{RGB}{57, 172, 57}
\definecolor{textred}{RGB}{200, 10, 10}
\definecolor{boxyellow}{HTML}{FAF5E6}
\definecolor{frameyellow}{HTML}{B7950B}
\definecolor{boxpurple}{HTML}{F4EFF6}
\definecolor{framepurple}{HTML}{6C3483}
\definecolor{boxblue}{HTML}{EEF4F8}
\definecolor{frameblue}{HTML}{2874A6}
\definecolor{boxgray}{HTML}{F0F2F3}
\definecolor{framegray}{HTML}{5D6D7E}
\definecolor{boxgreen}{HTML}{EAFaf1}
\definecolor{framegreen}{HTML}{196F3D}
\usepackage{pifont}

\usepackage[most]{tcolorbox}
\tcbset{
    center title,
    left=0pt,
    right=0pt,
    top=0pt,
    bottom=0pt,
    colframe=black,
    colback=gray!10,
    width=\linewidth,
    enlarge left by=0mm,
    boxsep=5pt,
    boxrule=1pt,
    arc=0pt,outer arc=0pt,
}

\newtcolorbox{promptbox}[1][]{
    enhanced,
    colback=white,
    colframe=black,
    fonttitle=\bfseries,
    title=Prompt,
    attach boxed title to top left={xshift=10pt, yshift*=-\tcboxedtitleheight/2},
    boxed title style={colback=black},
    top=12pt, bottom=10pt, left=10pt, right=10pt,
    #1
}

\tcbset{
    fancybox/.style={
        enhanced,
        breakable,
        arc=1mm,
        outer arc=1mm,
        boxrule=0pt,
        leftrule=1mm,
        toptitle=0.5mm,
        bottomtitle=0.5mm,
        fonttitle=\bfseries\sffamily\small,
        fontupper=\scriptsize,
        coltitle=white,
        drop shadow
    }
}

\newtcolorbox{thoughtbox}{
    fancybox,
    colback=boxyellow,
    colframe=frameyellow,
    coltitle=black,
    title=Thought
}
\newtcolorbox{userbox}{
    fancybox,
    colback=boxpurple,
    colframe=framepurple,
    title=User
}
\newtcolorbox{agentbox}{
    fancybox,
    colback=boxblue,
    colframe=frameblue,
    title=Agent
}
\newtcolorbox{outputbox}{
    fancybox,
    colback=boxgray, 
    colframe=framegray,
    coltitle=black,
    title=Execution Output
}
\newtcolorbox{solutionbox}{
    fancybox,
    colback=boxgreen,
    colframe=framegreen,
    title=Solution
}

\definecolor{codegreen}{rgb}{0.0, 0.5, 0.0}
\definecolor{codegray}{rgb}{0.4, 0.4, 0.4}
\definecolor{codepurple}{rgb}{0.50, 0, 0.50}
\definecolor{backcolour}{rgb}{0.97, 0.97, 0.97}

\lstdefinestyle{mystyle}{
    backgroundcolor=\color{backcolour},
    commentstyle=\color{codegreen},
    keywordstyle=\color{magenta},
    stringstyle=\color{codepurple},
    basicstyle=\ttfamily\scriptsize, 
    breakatwhitespace=false,
    breaklines=true,
    captionpos=b,
    keepspaces=true,
    numbers=none,              
    showspaces=false,
    showstringspaces=false,
    showtabs=false,
    tabsize=2,
    frame=single,
    rulecolor=\color{black!10}, 
    frameround=fttt,            
    upquote=true
}
\usepackage{graphicx}
\usepackage{subcaption}
\newcommand{\Id}{\mathrm{I}}

\title{
Bandit Simulation for Average Reward Inference
}

\author[1]{Samya Praharaj$^*$}
\author[2]{Chih-Yu Chang$^*$}
\author[1]{Koulik Khamaru}
\author[2]{Kelly W. Zhang}

\affil[1]{Rutgers University}
\affil[2]{Imperial College London}
\correspondingauthor{Correspondence to: kelly.zhang@imperial.ac.uk}

\codelink{https://github.com/kellywzhang/BanditSimulationForInference}

\begin{abstract}
Multi-arm bandit algorithms are increasingly used in online platforms, clinical trials, and social science experiments, but valid statistical inference on their performance remains an open challenge. After deploying bandits, a natural question is whether one can construct a confidence interval for its mean reward and assess whether it reliably outperforms a baseline policy. The total reward achieved in any single bandit deployment is random, and deploying a bandit twice on the same population typically yields different reward trajectories due to stochastic rewards. Standard statistical inference methods cannot be used because bandit algorithms introduce complex dependencies in the collected data, which violate the i.i.d. assumption underlying many classical approaches. Moreover, existing inference methods for adaptively collected data only apply to estimands that do not depend on the data-collection algorithm (such as the mean reward under a fixed action). We propose \bo{Bandit Simulation for Inference (BSI)}, a framework that fits a simulator of the bandit environment from observed data---either on-policy or off-policy---and uses it to estimate the mean reward under any evaluation policy, including adaptive blackbox algorithms. BSI formally propagates uncertainty in the estimated simulator parameters into the confidence interval construction. Furthermore, for BSI to be valid, it requires only weak exploration assumptions on the behavior policy and avoids importance weighting. We prove that BSI yields asymptotically valid confidence intervals, and demonstrate empirically that it maintains nominal coverage in settings where standard off-policy evaluation methods fail.
\end{abstract}

\begin{document}

\maketitle

\newenvironment{Itemize}{
    \begin{itemize}[leftmargin=*]
    \setlength{\itemsep}{0pt}
    \setlength{\topsep}{0pt}
    \setlength{\partopsep}{0pt}
    \setlength{\parskip}{1pt}}
{\end{itemize}}

\addtocontents{toc}{\protect\setcounter{tocdepth}{-1}}
\section{Introduction}
Adaptive sampling algorithms, such as bandits, are increasingly used across a variety of domains: digital platforms \citep{sajeev2021contextual,zhang2025impatient,feijer2025calibrated,zhang2024disco}, digital health \citep{liao2020personalized,trella2025deployed,kumar2024using,zhang2026replicable}, public policy \citep{kasy2021adaptive,offer2021adaptive}, public health \citep{pal2024improving,mate2023improved,bent2018novel}, and online education \citep{musabirov2025platform,cai2021bandit}. While much work has focused on designing algorithms that minimize regret, there is a growing need for valid post-deployment statistical inference. 
The total reward achieved in any single bandit deployment is random, and deploying a bandit twice on the same population typically yields different reward trajectories due to stochastic rewards. Thus, after deploying a bandit algorithm, a key question is: 
\begin{center}
\textit{What is the total expected reward of the algorithm on this population? Is it significantly greater than that of a simpler baseline or competing policy? Can we construct a valid confidence interval for it?}
\end{center}
Answering these questions is essential for deciding whether to roll out an algorithm more broadly. 

\subsection{Motivating Inference Problem}
\label{sec:motivating}
We consider a multi-armed bandit setting, where the environment $P$ comprises a collection of reward distributions $P \triangleq \{ P_a : a \in \MC{A} \}$ for arms $a \in \MC{A}$ with $|\MC{A}|< \infty$. For each $a \in \MC{A}$, the environment draws i.i.d. reward potential outcomes $R_1(a), \dots, R_T(a)$ from $P_a$. For each timestep $t \in [1 \colon T]$, a policy (or bandit algorithm) $\pi$ uses the history $\HH_{t-1} \triangleq \{ A_{t'}, R_{t'} \}_{t'=1}^{t-1}$ to form a distribution over $\MC{A}$, which is used to sample the action $A_t$. After selecting $A_t$, the reward $R_t \triangleq R_t(A_t)$ is revealed.  Formally, the policy $\pi$ forms the following action selection probabilities 
\begin{align*}
    \PP_{\pi}(A_t = a \mid \HH_{t-1}) = \pi(a \mid \HH_{t-1}).
\end{align*}
Note, our definition of policies $\pi$ encompasses adaptive, history-dependent, bandit algorithms. \textit{Static} policies are a special case where the action selection probabilities do not depend on $\HH_{t-1}$ or the timestep $t$:
\begin{align}
    \PP_{\pi}(A_t = a \mid \HH_{t-1}) = \PP_{\pi}(A_t = a) = \pi(a).
    \label{eqn:staticPolicy}
\end{align}
After deploying an adaptive (history-dependent) policy $\pi_1$ (e.g., Thompson sampling or $\epsilon$-greedy), we are interested in estimating and constructing confidence intervals for the \textit{value} of $\pi_1$:
\begin{align}
    \theta_T^* \triangleq \E_{P, \pi_1} \Bigg[ \frac{1}{T} \sum_{t=1}^T R_t \Bigg].
    \label{eqn:thetastar}
\end{align}
If $\pi$ was a static (non-adaptive) policy, the rewards $\{ R_t \}_{t=1}^T$ would be i.i.d. and the sample mean $\hat{\theta}_T \triangleq \frac{1}{T} \sum_{t=1}^T R_t$ would be an asymptotically Gaussian estimator for $\theta^*$. However, when $\pi$ is an adaptive, history-dependent policy, standard inference methods for i.i.d. are not applicable. In Figure \ref{fig:motivatingIssue}, we form the studentized t-statistic for the sample mean $\frac{\hat\theta_T - \theta_T^*}{ \hat{\sigma} / \sqrt{T}}$, where $\hat\sigma^2 \triangleq \frac{1}{T} \sum_{t=1}^T (R_t - \hat\theta_T)^2$. This distribution would be well-approximated by a t-distribution for non-adaptive policies $\pi$, but for common bandit algorithms, using this approximation can lead to significant Type-1 error inflation.

While there is significant literature on inference after adaptive sampling, it has focused on inferring parameters of the environment $P$ (e.g., parameters in a reward model or treatment effect) \citep{hadad2021confidence,zhang2020inference,zhang2021statistical,bibaut2021post,deshpande2018accurate,lin2025semiparametric,ying2023adaptive,khamaru2024inference}, which do not apply to inferring $\theta^*$. %
The fundamental difference with existing work is that the parameters of the environment $P$ do not depend on the adaptive algorithm ($\pi_1$); however, the mean reward $\theta^*$ depends on both the environment $P$ and the adaptive policy $\pi_1$. The works \citet{zhang2026replicable,zhang2022statistical} consider inference for the average reward under adaptive algorithms that satisfy smoothness conditions, but this excludes common algorithms like Thompson Sampling and $\epsilon$-greedy.

\begin{figure}[t]
    \vspace{-2mm}
    \centering
    \begin{subfigure}[t]{0.45\textwidth}
        \includegraphics[width=\textwidth,trim={1cm 0.5cm 1cm 0.5cm}]{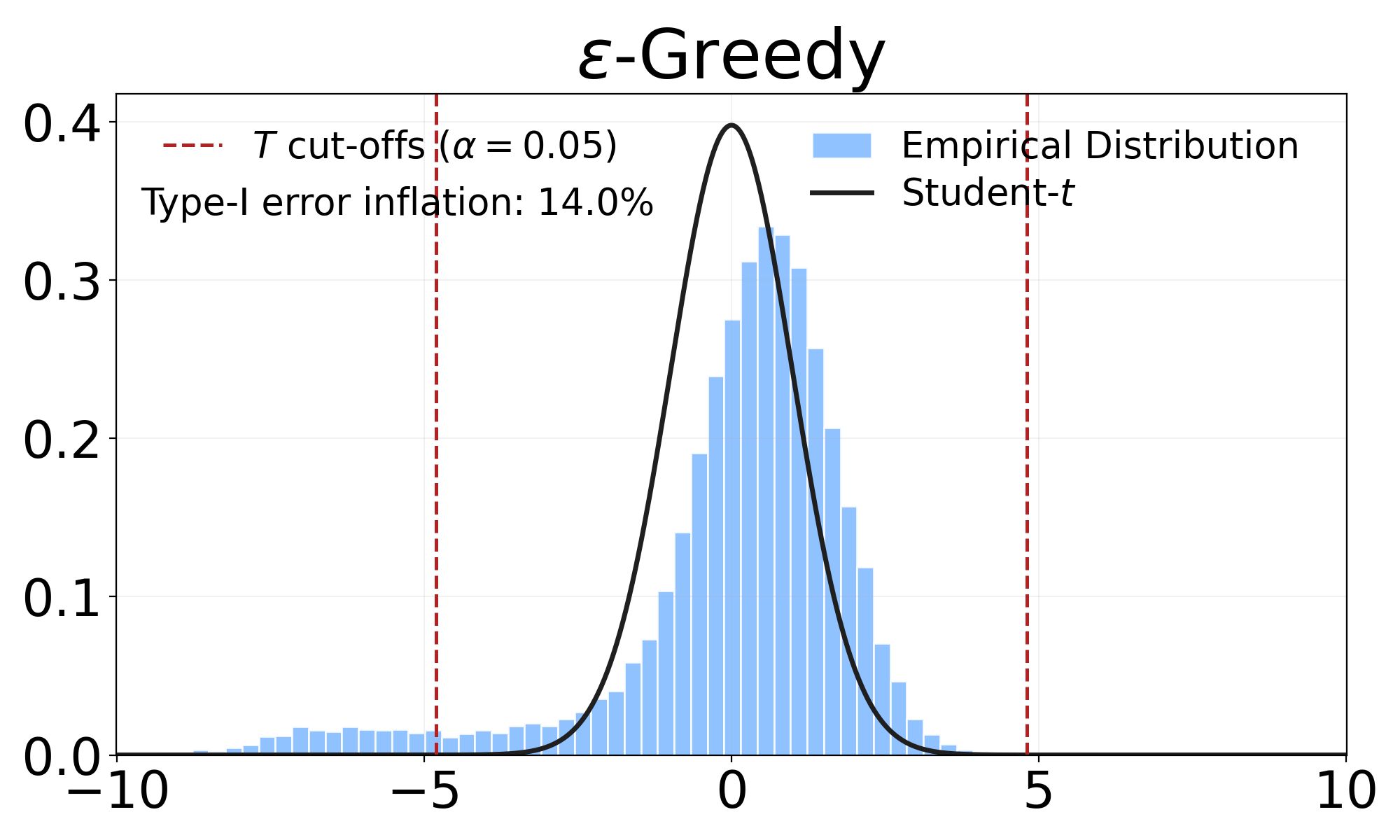}
    \end{subfigure}
    \hfill
    \begin{subfigure}[t]{0.45\textwidth}
        \includegraphics[width=\textwidth,trim={1cm 0.5cm 1cm 0.5cm}]{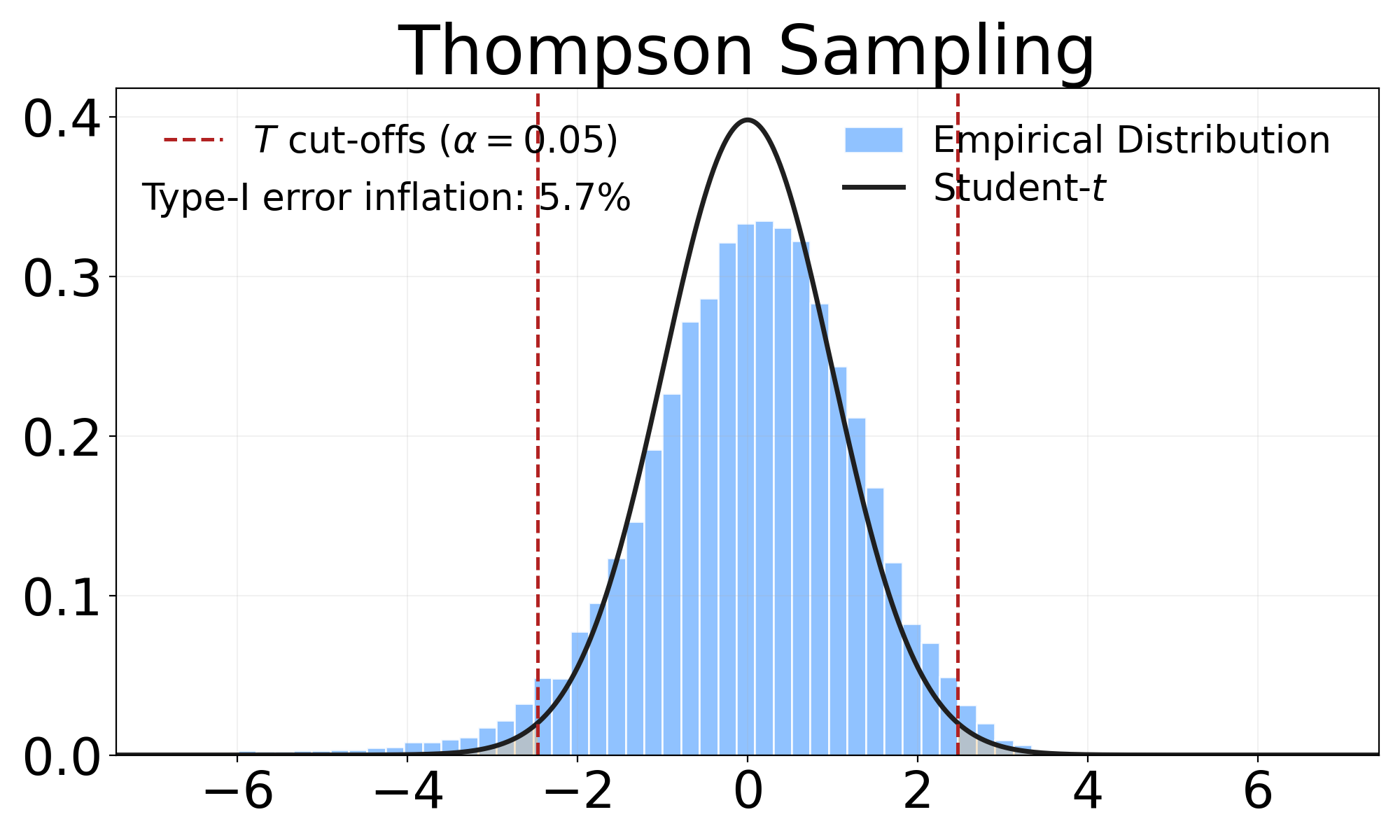}
    \end{subfigure}
    \caption{\bo{Histograms of the t-statistic for the Mean Reward under Thompson Sampling and $\epsilon$-Greedy.} Simulations ran in a three-armed bandit with arm means $(0.1,\,0.2,\,1.0)$, $\N(0, 1)$ reward noise, and horizon $T=100$ (20,000 Monte Carlo repetitions). We use the t-statistic for the mean reward $\frac{\hat\theta_T - \theta_T^*}{ \hat{\sigma} / \sqrt{T}}$, where $\hat{\theta}_T \triangleq \frac{1}{T} \sum_{t=1}^T R_t$ and $\hat\sigma^2 \triangleq \frac{1}{T} \sum_{t=1}^T (R_t - \hat\theta_T)^2$. }
    \label{fig:motivatingIssue}
    \vspace{-3mm}
\end{figure}

\subsection{General Inference Objective}
A more general way to formulate our inference objective is to assume we have access to an offline dataset $\MC{D}^{(\pi_0)}$ collected by a known behavior policy $\pi_0$ in environment $P$ for $\Toff$ timesteps:
\begin{align*}
    \MC{D}^{(\pi_0)} \triangleq \big\{ \big(A_t^{(\pi_0)}, R_t^{(\pi_0)} \big) \big\}_{t=1}^{\Toff}.
\end{align*}
We are still interested in inference for $\theta^*$ from display \eqref{eqn:thetastar}, the average reward under $\pi_1$ in environment $P$ run for $T$ time-steps. We use the off-policy dataset $\MC{D}^{(\pi_0)}$ to estimate and form confidence intervals for $\theta^*$. (The motivating problem from Section \ref{sec:motivating} is the special case in which $\pi_0 = \pi_1$ and $T = \Toff$.)
While this general inference objective may look like a classical off-policy evaluation problem, standard approaches \citep{thomas2016data,wang2017optimal} do not allow the evaluation policy $\pi_1$ to depend arbitrarily on the entire history $\HH_{t-1}$. 
Existing off-policy evaluation methods for adaptive behavior policies $\pi_0$ similarly assume the evaluation policy $\pi_1$ is static \citep{bibaut2021post,shen2026efficient,zhan2021off}. For further discussion of related work, see Section \ref{sec:relatedWork}. 

\subsection{Contributions}
We introduce \bo{Bandit Simulation for Inference (BSI)}, a framework for constructing asymptotically valid confidence intervals for $\theta^*$, the value of adaptive bandit policies. Our contributions are:
\begin{enumerate}[leftmargin=*]
    \item \bo{Simulation-based confidence intervals for policy value under adaptive data.} 
    BSI uses data collected by a known behavior policy $\pi_0$ to fit a simulator of the bandit environment, uses it to estimate the value of a target policy $\pi_1$, and explicitly accounts for simulator error in confidence interval construction. BSI works in both on-policy ($\pi_1 = \pi_0$) and off-policy ($\pi_1 \not= \pi_0$) settings, %
     avoids importance weighting, and supports the evaluation of arbitrary black-box, adaptive policies.
    \item \bo{Valid inference without stability assumptions.} We prove that BSI constructs asymptotically valid confidence intervals for $\theta^*$ in bandit environments with either parametric reward models (from a known family) or rewards with sub-Gaussian noise. In contrast to prior approaches, our guarantees do not require stability or smoothness conditions on the behavior or evaluation policy.
    \item \bo{Reliable coverage where standard methods underperform.} We empirically evaluate in two synthetic environments (Gaussian and non-Gaussian sub-Gaussian rewards) and a semi-synthetic environment derived from a randomized controlled trial of financial incentives for mobile phone survey participation in Uganda \citep{gibson2022promised}. BSI maintains valid coverage under adaptive evaluation policies, while inference methods for i.i.d. data and existing off-policy evaluation methods often exhibit significant undercoverage.
\end{enumerate}

\section{Related Work}
\label{sec:relatedWork}

\bo{Inference after adaptive sampling.}
Our goal is to construct confidence intervals for the average reward $\theta_T^* = \mathbb{E}_{P,\pi_1}\bigl[\frac{1}{T}\sum_{t=1}^T R_t\bigr]$ under an adaptive, history-dependent policy~$\pi_1$.  This estimand depends on both the environment~$P$ and the policy~$\pi_1$, a feature that distinguishes our problem from most existing work. A substantial body of work addresses inference for parameters of the reward distribution~$\{ P_a : a \in \MC{A} \}$---such as arm means, arm-mean contrasts, or regression coefficients---from adaptively collected data \citep{bibaut2025demystifying}. Much of this literature constructs asymptotically valid confidence intervals via asymptotic Gaussian assumptions. Under stability or smoothness assumptions on the algorithm, standard estimators are asymptotically Gaussian \citep{lai1982least,khamaru2024inference,han2024ucb,halder2505stable,yan2026optimism,praharaj2025avoiding,fan2025statistical,guo2025statistical,zhang2022statistical,zhang2026replicable}. Since many common bandit algorithms do not satisfy such stability conditions, adjusted versions of estimators that incorporate adaptive weighting are used to derive asymptotic normality results under general bandit algorithms \citep{hadad2021confidence,zhang2020inference,zhang2021statistical,deshpande2018accurate,khamaru2021near}. A separate line of work constructs confidence regions for parameters of $\{ P_a : a \in \MC{A} \}$ via concentration inequalities, yielding guarantees that hold non-asymptotically and uniformly over time \citep{abbasi2011improved,howard2021time,waudby2024time}. 

Recent work by \citet{cho2025simulation} also uses observed data to fit a environment simulator to use for inference after adaptive sampling. Major differences from our work are \textit{(i) The estimand:} they construct confidence intervals for the mean reward of a fixed arm, while we are interested in the value of an adaptive algorithm; \textit{(ii) Applicability to different bandit algorithms:} they only provide theoretical guarantees for their approach under two bandit algorithms (explore-then-commit and UCB), while our approach can be used to evaluate the value of any adaptive algorithm (as long as the behavior policy explores sufficiently); \textit{(iii) Simulator estimation approach:} they require estimators of the simulator parameters to be estimated optimistically with a bonus term that scales at a specific rate, while we allow the simulator parameters to be estimated using any asymptotically Gaussian estimator.

\bo{Off-policy evaluation (OPE).}
OPE uses data collected under a behavior policy~$\pi_0$
to estimate the value of a different target policy $\pi_1$. If a static behavior policy $\pi_0$ is used, many approaches including importance weighting, doubly robust estimators, and bootstrapping can be used \citep{horvitz1952generalization,robins1994estimation,chernozhukov2018double,wang2017optimal,kallus2020double,uehara2022review,ouyang2026bayesian,hanna2017bootstrapping,dai2020coindice}. If the behavior policy is adaptive, but one can run multiple i.i.d. replications of the algorithm \citep{ghosh2024did,liao2020personalized}, one can use standard asymptotics for i.i.d. data across trajectories to infer the value of the policy. For OPE with adaptive behavior policies $\pi_0$, adaptive weighting methods \citep{bibaut2021post,zhan2021off} and empirical likelihood approaches \citep{karampatziakis2020empirical} have been developed, but these approaches only apply to static evaluation policies $\pi_1$. In this work, we are interested in inference when both the evaluation policy and behavior policies may both be adaptive.

\section{Bandit Simulation for Inference with Parametric Reward Distributions}\label{sec:theory}
We now formally introduce our problem and discuss our \textbf{Bandit Simulation for Inference (BSI)} approach (Algorithm~\ref{alg:sbi_bandit}). We begin with the case in which reward distributions belong to a known parametric family. Specifically, the reward distribution $P_a$ is parameterized by an unknown, low-dimensional parameter $\lambda_a^* \in \mathbb{R}^d$. We use $\lambda^\ast = (\lambda_a^\ast)_{a \in \MC{A}} \in \Lambda$ to denote the vector containing the parameters of all $|\MC{A}|$ arms, and use the shorthand $P_a \triangleq P_{\lambda_a^*}$. This encompasses many standard reward models:
\begin{enumerate*}[label=(\roman*)] 
	\item if $P_a$ is Gaussian with mean $\mu_a^*$ and variance $(\sigma_a^*)^2$, then $\lambda_a^* = \big( \mu_a^*, (\sigma_a^*)^2 \big) \in \mathbb{R} \by \mathbb{R}_{>0}$;
	\item if $P_a$ is Bernoulli with mean $\mu_a^*$, then $\lambda_a^* = \mu_a^* \in [0, 1]$;
    \item if $P_a$ is Poisson with mean $\mu_a^*$, then $\lambda_a^* = \mu_a^* \in \mathbb{R}_{>0}$.
\end{enumerate*}

Given a horizon $T > 0$ and a (possibly history-dependent bandit or black-box) \emph{target policy} $\pi_1$, we construct a confidence interval for the expected average reward of deploying $\pi_1$ for $T$ rounds in the environment $P_{\lambda^\ast} \triangleq \{ P_{\lambda^\ast_a} \}_{a \in \MC{A}}$:
\begin{align*}
	\theta^\ast_T = f_T(\lambda^\ast, \pi_1) \triangleq \E_{P_{\lambda^\ast}, \pi_1} \bigg[  \frac{1}{T} \sum_{t=1}^T R_t \bigg]
\end{align*}
Above, we use the function $f_T(\lambda, \pi)$ to denote the value of policy $\pi$ deployed in environment $\lambda$ for $T$ timesteps.

Our only assumption on the target policy $\pi_1$ is that we can execute it in a simulated environment. 
If we knew the true parameter values $\lambda^* = \left\lbrace \lambda_a^\ast\right\rbrace_{a \in \MC{A}}$, we could simply estimate $f_T(\lambda^\ast, \pi_1)$ by running $\pi_1$ in the environment $P_{\lambda^*}$ repeatedly using Monte-Carlo (MC) simulations. However, we are only able to form estimates $\hat\lambda$ of $\lambda^*$ using an offline dataset $\MC{D}^{(\pi_0)}$, a \emph{single trajectory} of length $\Toff$ drawn using (a possibly bandit) \emph{behavior policy} $\pi_0$. 

To estimate $\theta^*_T$, we use the plugin estimator $\hat{\theta}_T = f_T(\hat{\lambda}, \pi_1)$. Given $\hat{\lambda}$, we estimate  $f_T(\hat{\lambda}, \pi_1)$ by running a large number of MC simulations. For our theoretical results, we ignore the MC error in estimating $f_T(\hat{\lambda}, \pi_1)$ given $\hat{\lambda}$ since for any fixed $\lambda$, $f_T(\lambda, \pi_1)$ can be estimated arbitrarily accurately with enough MC simulations. The dominant source of uncertainty is in the estimation of $\lambda^*$. The key idea of BSI is to account for how the uncertainty  in $\hat{\lambda}$ 
propagates into uncertainty in $f_T(\hat{\lambda}, \pi_1)$.

\subsection{Asymptotic Normality under Static (or Stable) Behavior Policies}
\label{sec:stable}
We now formalize the assumptions we use to show asymptotic normality for our plugin estimator $\hat{\theta}_T$. In this section, we consider the setting where the behavior policy $\pi_0$ is static (non-adaptive) or converges uniformly to a limiting policy; formally we require that Assumption~\ref{assump:lambdaNormality} below holds. We discuss the extension to adaptive behavior policies in Section \ref{sec:adaptiveBehavior}. %

\begin{assumption}[Asymptotic Normality $\hat{\lambda}$]
	\label{assump:lambdaNormality}
    $\sqrt{\Toff}  ( \hat{\lambda} - \lambda^* )
	\Dto \N ( 0,  \Sigma_{\lambda^*,\pi_0} )$ uniformly over the underlying environment $P_{\lambda^*, \pi_0} \in \{ P_{\lambda, \pi_0} : \lambda \in \Lambda \}$. Also, %
    $\sup_{\lambda^* \in \Lambda} \| \Sigma_{\lambda^*,\pi_0} \|_2 < \infty$.
\end{assumption}

Assumption \ref{assump:lambdaNormality} states that one is able to form estimates $\hat{\lambda} \triangleq ( \hat{\lambda}_a )_{a \in \MC{A}}$ of $\lambda^* \triangleq ( \lambda_a^* )_{a \in \MC{A}}$ that are asymptotically normal. 
When $\pi_0$ is a static policy, i.e. satisfies \eqref{eqn:staticPolicy}, the data $D^{(\pi_0)}$ consists of i.i.d. tuples. Consequently, Assumption \ref{assump:lambdaNormality} holds under standard regularity conditions for M-estimators on i.i.d. data when $\hat{\lambda}$ is the maximum likelihood estimator or the minimizer of a loss function and $\pi_0(a) > 0$ for all arms $a \in \MC{A}$; see Theorem 5.23 of \cite{van2000asymptotic}. 
Assumption \ref{assump:lambdaNormality} can also hold if $\pi_0$ is adaptive but uniformly converges to a static limiting policy as $\Toff \to \infty$; note though that common bandit algorithms (e.g., Thompson Sampling, $\epsilon$-greedy) do not satisfy such uniform convergence conditions, but certain specialized algorithms may (e.g., Boltzmann sampling)~\citep{lai1982least,zhang2020inference,zhang2026replicable,deshpande2018accurate}.

\vspace{-1mm}
\begin{algorithm}[b!]
\caption{Bandit Simulation for Inference (BSI)}
\label{alg:sbi_bandit}
\begin{algorithmic}[1]
\Require Offline data $\mathcal{D}^{(\pi_0)}$, target policy $\pi_1$, horizon $T$, number of Monte Carlo repetitions $m$
\State \texttt{\textbf{\textcolor{blue}{ // Fit simulation environment parameters (Assumption \ref{assump:lambdaNormality} or \ref{assump:lambdaNormalityAdaptive})}}}
\State Use $\mathcal{D}^{(\pi_0)}$ to form $\hat{\lambda} = (\hat{\lambda}_a)_{a \in \MC{A}}$
\For{$i=1,\dots,m$} \qquad \texttt{\textbf{\textcolor{blue}{ // Repeatedly run $\pi_1$ in fitted simulation environment}}}
    \State Initialize simulated history ${H}^{(i)}_0 \gets \emptyset$ 
    \For{$t=1,\dots,T$}
        \State Sample action: $A_t^{(i)} \sim \pi_1(\cdot \mid {H}^{(i)}_{t-1})$
        \State Sample reward: $R_t^{(i)} \sim P_{\hat{\lambda}_{ A_t^{(i)}}} (\cdot \mid A_t^{(i)})$
        \State Update history: ${H}^{(i)}_t \gets {H}^{(i)}_{t-1} \cup \big\{ \big(A_t^{(i)},R_t^{(i)} \big) \big\}$
    \EndFor
    \State $\bar{R}_T^{(i)} \gets \frac{1}{T}\sum_{t=1}^T R_t^{(i)}$
    \State $\nabla \hat{f}_T^{(i)} \gets \frac{1}{T} \sum_{t=1}^T \nabla_{\hat{\lambda}} \log P_{\hat{\lambda}_{A_t^{(i)}}} (R_t^{(i)} \mid A_t^{(i)}) \cdot \sum_{t'=t}^T R_{t'}^{(i)}$ %
\EndFor
\State $\hat{\theta}_T \gets \frac{1}{m}\sum_{i=1}^m \bar{R}_T^{(i)}$ \qquad \texttt{\textbf{\textcolor{blue}{ // Average reward estimator}}}
\State $\nabla \hat{f}_T \gets \frac{1}{m} \sum_{i=1}^m \nabla \hat{f}_T^{(i)}$ \qquad \texttt{\textbf{\textcolor{blue}{ // Estimator of gradient $\nabla_\lambda f(\hat{\lambda}, \pi_1)$ from \eqref{eqn:fgrad}} }}
\State \Return $\hat{\theta}_T$ and $\nabla \hat{f}_T$ \qquad \texttt{\textbf{\textcolor{blue}{ // Estimator \& gradient estimate}}}
\end{algorithmic}
\end{algorithm}

\begin{assumption}[Differentiability]
    \label{assump:differentiable}
 $f_T(\lambda, \pi_1)$ is uniformly, continuously differentiable in $\lambda \in \Lambda$, i.e.,
    \begin{align}
        \sup_{\lambda \in \Lambda } \;\; \sup_{0 < \|h\|_2 \leq \delta, \;\;  
        \lambda + h \in \Lambda} \frac{ | f_T ( \lambda + h, \pi_1) - f_T(\lambda, \pi_1) - h^\top \nabla f_T(\lambda, \pi_1) | }{ \|h\|_2} \to 0 \qquad \TN{as~} \delta \to 0.
        \label{eqn:uniformlyDiff}
    \end{align}
\end{assumption} 
While one might think at first glance that to satisfy Assumption \ref{assump:differentiable} that we might need the policy $\pi_1$ itself to satisfy certain smoothness or differentiability conditions. Interestingly, we show in Section \ref{sec:differentiability} that $f_T$ is differentiable for any policy $\pi_1$ as long as the reward density function $P_\lambda$ satisfies mild differentiability and moment conditions (e.g., all exponential families satisfy such conditions).

We now state our main result for static behavior policies, which shows that the estimator for the expected average reward $\theta_T^*$, obtained from the BSI Algorithm~\ref{alg:sbi_bandit} is asymptotically normal. Consequently, we can construct an asymptotically valid confidence interval for $\theta_T^*$.

\begin{theorem}[Bandit Simulation for Inference: Static $\pi_0$]
    \label{thm:normality}
    Under Assumptions \ref{assump:lambdaNormality}-\ref{assump:differentiable}, $\hat{\theta}_T$ satisfies
    \begin{align*}
        \sqrt{ \Toff } (\hat{\theta}_T - \theta_T^*) \Dto \N \Big( 0, \, \nabla_\lambda f_T(\lambda^*, \pi_1)^\top \Sigma_{\lambda^*,\pi_0} \nabla_\lambda f_T(\lambda^*, \pi_1) \Big) \quad \TN{as} \quad \Toff \to \infty,
    \end{align*}
    uniformly over the underlying  $P_{\lambda^*, \pi_0} \in \{ P_{\lambda, \pi_0} : \lambda \in \Lambda \}$.
    Thus, assuming a non-degenerate limiting variance, i.e., $\inf_{\lambda^* \in \Lambda} \nabla_\lambda f_T(\lambda^*, \pi_1)^\top \Sigma_{\lambda^*, \pi_0} \nabla_\lambda f_T(\lambda^*, \pi_1) > 0$, then for any %
    $\alpha \in (0, 1)$, %
    \begin{align}
        \label{eqn:CI}
        \lim_{\Toff \to \infty} \inf_{\lambda^* \in \Lambda} \PP_{P_{\lambda^*}, \pi_0} \bigg( \theta_T^* \in \bigg[ \hat{\theta}_T \pm \frac{ z_{1-\alpha/2} }{ \sqrt{ \Toff} } \sqrt{ \nabla_\lambda f_T(\hat{\lambda}, \pi_1)^\top \hat\Sigma_{\hat{\lambda},\pi_0} \nabla_\lambda f_T(\hat{\lambda}, \pi_1)  } 
        \bigg] \bigg) \geq 1-\alpha,
    \end{align}
    where $\hat\Sigma_{\hat{\lambda},\pi_0}$ is  any  positive definite matrix satisfying $\hat\Sigma_{\hat{\lambda},\pi_0} \Pto \Sigma_{\lambda^*, \pi_0}$, uniformly over $\lambda^\ast \in \Lambda$.   
\end{theorem}

In practice, we estimate gradients $\nabla_\lambda f_T(\hat{\lambda}, \pi_1)$ using a large number of MC simulations (see Section \ref{sec:differentiability} and Algorithm \ref{alg:sbi_bandit}). 
Note that Theorem \ref{thm:normality} essentially places \textit{no restrictions} on the evaluation policy $\pi_1$, meaning $\pi_1$ can be adaptive or even produced by a black box model. We only require the ability to execute $\pi_1$ in a simulated environment.  Theorem \ref{thm:normality} holds by Lemma \ref{lemma:linearization}, a uniform linearization result akin to a uniform version of the Delta method \citep[Theorem 3.8]{van2000asymptotic}; see Appendix \ref{app:proofNormality}.

\subsection{Differentiability of \texorpdfstring{$f_T$}{f} }
\label{sec:differentiability}

We now formalize the conditions under which differentiability Assumption \ref{assump:differentiable} holds and discuss the estimation of the gradients $\nabla_\lambda f_T(\lambda^*, \pi_1)$. The proof of Proposition \ref{prop:differentiablef} is deferred to Appendix \ref{app:differentiability}.

\begin{prop}[Derivative of $f_T$]
    \label{prop:differentiablef} The function
    $f_T(\lambda, \pi_1)$ for finite $T$ satisfies Assumption \ref{assump:differentiable} for an open set $\Lambda$ with a compact closure, if for any $\lambda %
    \in \Lambda$ the following conditions are satisfied:
    \begin{itemize}
        \item[(a)] $\E_{P_{\lambda_a}} \big[ R_t(a)^2 \big] < \infty$ and $\E_{\lambda_a} \big[ \left\| \nabla_{\lambda_a} \log P_{\lambda_a} \left( R_t(a) \right) \right\|_1^2 \big] < \infty$
        \item[(b)]  $P_{\lambda_a} \left( R_t(a) \right)$ and $\nabla_{\lambda_a} \log P_{\lambda_a} \left( R_t(a) \right)$ are almost surely continuous in $\lambda_a$.
    \end{itemize}
    Moreover, the derivative $\nabla_{\lambda} f_T(\lambda, \pi_1)$ can be written as follows:
    \begin{align}
        \label{eqn:fgrad}
        \nabla_{\lambda} f_T(\lambda, \pi_1) 
        = \frac{1}{T} \sum_{t=1}^{T} \E_{P_{\lambda}, \pi_1} \bigg[ \nabla_{\lambda} \log P_{\lambda_a} (R_{t} \mid A_{t}) \cdot \sum_{t'=t}^T  R_{t'} \bigg].
    \end{align}
\end{prop}
Note that the above expression is similar to the \textit{policy gradient} \citep{sutton2018reinforcement}, but differs because we are not differentiating with respect to parameters in the policy $\pi_1$, but parameters in the underlying bandit environment $P_\lambda$. Differentiating with respect to the simulator parameters has been used in other areas of RL (especially robotics) to improve the efficiency of policy optimization \citep{freeman2021brax,de2018end,hu2019chainqueen}. We believe this is the first use of these simulator gradients for the purposes of statistical inference on adaptively collected data.

In practice, for any fixed $\lambda$, we can approximate the gradient $\nabla_{\lambda} f_T(\lambda, \pi_1)$ with arbitrary accuracy with sufficient MC simulations. Below we provide expressions for $\nabla_{\lambda} f_T(\lambda, \pi_1)$ for Bernoulli and Gaussian rewards;  
see Appendix \ref{app:exampleDistributions} for the general case for natural exponential families.

\begin{example}[Bernoulli rewards]
    \label{ex:bernoulli}
    If $\lambda_a = \mu_a$ is the mean of the Bernoulli rewards for action $a$, then
    \begin{align*}
        \nabla_{\lambda_a} f_T(\lambda, \pi_1) 
        &= \frac{1}{T} \sum_{t=1}^{T} \E_{P_{\lambda}, \pi_1} \bigg[ \mathbbm{1}(A_t = a) \bigg( 1 + \frac{R_t - \mu_{a}}{\mu_{a} (1 - \mu_{a})} \sum_{t'=t+1}^T R_{t'} \bigg) \bigg]
    \end{align*}
\end{example}

\begin{example}[Gaussian rewards]
    \label{ex:gaussian}
    If $\lambda_a = (\mu_a, \sigma_a^2)$ is the mean and variance of a Gaussian, then
    \begin{align*}
        \nabla_{\lambda_a} f_T(\lambda, \pi_1) 
        &= \frac{1}{T} \sum_{t=1}^{T} \E_{P_{\lambda}, \pi_1} \begin{bmatrix}
            \mathbbm{1}(A_t = a) \big( 1 + \frac{R_t - \mu_a}{\sigma_a^2} \sum_{t'=t+1}^T R_{t'} \big) \\
            \mathbbm{1}(A_t = a) \big( -\frac{1}{2 \sigma_a^2} + \frac{(R_t - \mu_a)^2}{ 2 \sigma_a^4 } \big) \sum_{t'=t+1}^T R_{t'}.
        \end{bmatrix}
    \end{align*}
    If the reward variance for arm $a$, $\sigma_a^2$, is known, then $\lambda_a = \mu_a$ and $\nabla_{\lambda_a} f_T(\lambda, \pi_1) = \frac{1}{T} \sum_{t=1}^{T} \E_{P_{\lambda}, \pi_1} \big[ \mathbbm{1}(A_t = a) \big( 1 + \frac{R_t - \mu_a}{\sigma_a^2} \sum_{t'=t+1}^T R_{t'} \big) \big]$.
\end{example}

\subsection{Extension to Policy Contrasts}

In many applications, the quantity of interest is not the expected reward of a single policy, but whether one policy outperforms another. For instance, in clinical trials or social science experiments \citep{pal2024improving,mate2023improved}, one may wish to assess whether a new adaptive policy improves over an existing deployed policy---without actually deploying the new policy.
We extend Theorem~\ref{thm:normality} to infer \textit{policy contrasts},  the difference in expected rewards between two (possibly adaptive, black-box) policies $\pi_1$ and $\pi_2$:
\begin{align*}
    \delta_T^* = g_T(\lambda^*, \pi_1, \pi_2) \triangleq \E_{P_{\lambda^\ast}, \pi_1} \bigg[ \frac{1}{T} \sum_{t=1}^T R_t \bigg] - \E_{P_{\lambda^\ast}, \pi_2} \bigg[ \frac{1}{T} \sum_{t=1}^T R_t \bigg].
\end{align*}
We estimate the policy contrast $\delta^\ast_T$ with the plugin estimator $\hat{\delta}_T \triangleq f_T(\hat\lambda, \pi_1) - f_T(\hat\lambda, \pi_2)$. %
The asymptotic normality result for $\hat{\delta}_T$ is a direct corollary of Theorem \ref{thm:normality}; see Appendix~\ref{app:policyContrast} for the proof.
\begin{corollary}[Policy Contrasts]
    \label{corr:contrast}
    Under Assumptions \ref{assump:lambdaNormality} and~\ref{assump:differentiable}, the plug-in estimator $\hat{\delta}_T$ %
    satisfies  
    \begin{align*}
        \sqrt{ \Toff } (\hat{\delta}_T - \delta_T^*) \Dto \N \Big( 0, \, \nabla_\lambda  g_T(\lambda^*, \pi_1, \pi_2)^\top \Sigma_{\lambda^*,\pi_0} \nabla_\lambda g_T(\lambda^*, \pi_1, \pi_2) \Big) \quad \TN{as} \quad \Toff \to \infty,
    \end{align*}
    uniformly over the underlying  $P_{\lambda^*, \pi_0} \in \{ P_{\lambda, \pi_0} : \lambda \in \Lambda \}$.  
    Thus, assuming a non-degenerate limiting variance, i.e., $\inf_{\lambda^* \in \Lambda} \nabla_\lambda g_T(\lambda^*, \pi_1, \pi_2)^\top \Sigma_{\lambda^*, \pi_0} \nabla_\lambda g_T(\lambda^*, \pi_1, \pi_2) > 0$, then for any $\alpha \in (0, 1)$,%
    \begin{align*}
        \lim_{\Toff \to \infty} \inf_{\lambda^* \in \Lambda} \PP_{P_{\lambda^*}, \pi_0} \bigg( \delta_T^* \in \bigg[ \hat{\delta}_T \pm \frac{ z_{1-\alpha/2} }{ \sqrt{ \Toff} } \sqrt{ \nabla_\lambda g_T(\hat{\lambda}, \pi_1, \pi_2)^\top \hat\Sigma_{\hat{\lambda},\pi_0} \nabla_\lambda g_T(\hat{\lambda}, \pi_1, \pi_2)  } 
        \bigg] \bigg) \geq 1-\alpha,
    \end{align*}
where $\hat\Sigma_{\hat{\lambda},\pi_0}$ is  any  positive definite matrix satisfying $\hat\Sigma_{\hat{\lambda},\pi_0} \Pto \Sigma_{\lambda^*, \pi_0}$, uniformly over $\lambda^\ast \in \Lambda$. 
\end{corollary}

\subsection{Extension to Adaptive Behavior Policies}
\label{sec:adaptiveBehavior}
When the behavior policy $\pi_0$ is a history-dependent, bandit algorithm Assumption \ref{assump:lambdaNormality} in general does not hold for standard M-estimators \citep{deshpande2018accurate,hadad2021confidence,zhang2020inference}. Instead, adaptively weighted versions of loss minimizers (M-estimators) are asymptotically normal \citep{zhang2021statistical} and satisfy Assumption \ref{assump:lambdaNormalityAdaptive} below. 
\begin{assumption}[Self-Normalized Asymptotic Normality of $\hat{\lambda}$]
	\label{assump:lambdaNormalityAdaptive}
    For a sequence of random positive definite matrices $V_{\Toff}$,  with
    $\lambda_{\min}(V_{\Toff}) \rightarrow \infty$  w.p. $1$ as $\Toff \to \infty$, the estimator $\hat\lambda$ satisfies 
    \begin{align}
        V_{\Toff} \big( \hat{\lambda} - \lambda^* \big)
	   \Dto \N \big( 0, \Id \big) \quad \TN{as} \quad \Toff \to \infty,
       \label{eqn:lambdaNormalityAdaptive}
    \end{align}
    uniformly over the underlying environment $\lambda^* \in \Lambda$.
\end{assumption}
Assumption~\ref{assump:lambdaNormalityAdaptive} is satisfied by adaptively weighted M-estimators when (i) the M-estimation function (e.g., loss function that defines the estimator) meets standard regularity conditions that are typical for i.i.d. data and (ii) $\pi_0$ explores sufficiently. The key exploration requirement is $\pi_0(a \mid \MC{H}_{t-1}) \geq c$ w.p.\ $1$ for some constant $c > 0$. This condition can be further weakened to $\pi_0(a \mid \MC{H}_{t-1}) \geq c_{\Toff}$ w.p.\ 
$1$ with $c_{\Toff}/\Toff \to 0$, when $\hat{\lambda}$ is a least squares estimator; see Theorem~1 in 
\citet{zhang2021statistical}.

\begin{theorem}[Simulation-Based Inference: Adaptive $\pi_0$]
    \label{thm:asymptoticCIAdaptive}
    Under Assumptions \ref{assump:differentiable}-\ref{assump:lambdaNormalityAdaptive}, we can construct the following asymptotically valid confidence intervals for $\theta^*$ with any $\alpha \in (0, 1)$,
    \begin{align}
        \label{eqn:CIadapt}
        \lim_{\Toff \to \infty} \inf_{\lambda^* \in \Lambda} \PP_{P_{\lambda^*}, \pi_0} \Big( \theta_T^* \in \Big[ \hat{\theta}_T \pm \sqrt{ \chi_{\TN{dim}(\lambda), 1-\alpha}^2 } \cdot \big\| (V_{\Toff}^{\top} V_{\Toff})^{-1/2} \nabla f_T( \hat\lambda, \pi_1 ) \big\|_2 \Big] \Big) \geq 1-\alpha.
    \end{align}
\end{theorem}
Since $V_{\Toff}$ is a random matrix, $\hat{\lambda}$ is not approximately 
Gaussian with covariance $V_{\Toff}$, and the Delta method approach used in Theorem~\ref{thm:normality} does not apply. Instead, we construct confidence intervals for $f(\hat{\lambda}, \pi_1)$ via a projection. The resulting intervals are wider than those of 
Theorem~\ref{thm:normality} by a factor of $\sqrt{\chi^2_{\TN{dim}(\lambda), 
1-\alpha}} / z_{1-\alpha/2}$, which equals $1$ only when $\TN{dim}(\lambda) 
= 1$; see Appendix~\ref{app:proofCIAdaptive} for details.

\newcommand{\Gn}{\mathrm{G}}

\section{Bandit Simulation for Inference: Sub-Gaussian Reward Distributions}
\newcommand{\regret}{\mathcal{R}}  
\label{sec:subgaussian}

We now consider a bandit environment $P = \{ P_a \}_{a \in \MC{A}}$, which is not from a known parametric family, but is simply known to be sub-Gaussian. We use $\mu_a^* = \int r \cdot P_a(r) dr$ to refer to the mean of $P_a$ and  $(\sigma_a^*)^2 \triangleq \inf \big\{ \sigma^2 > 0 : \E\left[ e^{s(R_t(a) - \mu_a^*)} \right] \leq e^{s^2 \sigma^2 / 2}, \; \forall s \in \mathbb{R} \big\}$ denotes the sub-Gaussian variance proxy for $P_a$. Since $P$ is not from a known family, we cannot directly simulate from $P$ even if we knew the values of $\mu_a^\ast$ and $\sigma_a^\ast$. 

Our idea is to approximate the $P$ with a Gaussian environment $P^{\Gn}$, where $P^{\Gn}$ is $\mathcal{N}(\mu^\ast_a, (\sigma_a^*)^2)$. We denote the average reward under $\pi_1$ in the Gaussian environment $P^{\Gn}$ using
\begin{align*}
    \theta^\Gn_T  \triangleq \E_{P^{\Gn}, \pi_1} \bigg[ \frac{1}{T} \sum_{t=1}^T R_t \bigg].
\end{align*}

\bo{Known Variance Proxies $\big( (\sigma_a^*)^2 \big)_{a \in \MC{A}}$.}
We begin with the setting in which the sub-Gaussian variance proxy parameters $\big( (\sigma_a^*)^2 \big)_{a \in \MC{A}}$ are known. We note that Algorithm \ref{alg:sbi_bandit} can still be applied in this setting to form asymptotically valid confidence intervals for $\theta_T^G$. The theoretical results in Section \ref{sec:theory} were presented in the setting in which the real data $\MC{D}^{(\pi_0)}$ was collected from an environment with a known parametric model ($P^G$) and the estimand was the expected reward of $\pi_1$ in that environment ($\theta_T^G$). However, Theorems \ref{thm:normality} and \ref{thm:asymptoticCIAdaptive} can still be applied here to form asymptotically valid confidence intervals for $\theta_T^G$ even when the real data $\MC{D}^{(\pi_0)}$ is collected from an environment $P$ with matching mean rewards and sub-Gaussian variance proxies.

Specifically, we use Algorithm \ref{alg:sbi_bandit} to fit the Gaussian reward environment $P^G \triangleq P_{\lambda^*}^G = \big( P_{\lambda_a^*}^G \big)_{a \in \MC{A}}$ where $\lambda^* = (\lambda_a^*)_{a \in \MC{A}} = (\mu_a^*)_{a \in \MC{A}}$ and $P_{\lambda_a^*}^G$ is $\N \big( \mu_a^*, (\sigma_a^*)^2 \big)$; recall the sub-Gaussian variance proxy parameters $\big( (\sigma_a^*)^2 \big)_{a \in \MC{A}}$ are known. Theorems \ref{thm:normality} and \ref{thm:asymptoticCIAdaptive} can still be applied here to form asymptotically valid confidence intervals for
\begin{align*}
    \theta_T^G = f_T(\lambda^*, \pi_1) \triangleq \E_{P^G_{\lambda^*}, \pi_1} \left[ \frac{1}{T} \sum_{t=1}^T R_t \right]
\end{align*}
as the Theorems only require that $f_T$ satisfies differentiability conditions (Assumption \ref{assump:differentiable}) and that $\hat\lambda$ is an asymptotically normal estimator of $\lambda^*$ (Assumption \ref{assump:lambdaNormality} and \ref{assump:lambdaNormalityAdaptive}). Since $\lambda^* = (\mu_a^*)_{a \in \MC{A}}$ and the mean rewards $(\mu_a^*)_{a \in \MC{A}}$ are the same in both environments $P$ and $P^G_{\lambda^*}$, thus the $\hat\lambda$ asymptotic normality conditions can still be satisfied using the estimation approaches described earlier in Sections \ref{sec:stable} and \ref{sec:adaptiveBehavior}. \textit{The key question is, how close is $\theta^\Gn_T$ to our true estimand $\theta^*_{T} \triangleq \E_{P, \pi_1} \big[ \frac{1}{T} \sum_{t=1}^T R_t \big]$?}

Even though the two environments $P$ and $P^{\Gn}_{\lambda^*}$ have the same mean and sub-Gaussian parameters, the variance and higher moments of their reward distributions may differ. Note though that $\big| \theta_T^* - \theta^\Gn_T \big|$ must be small when $\pi_1$ has similar average rewards in both environments $P$ and $P^{\Gn}$, i.e., when the potential differences in the higher moments of the reward distributions in $P$ and $P^{\Gn}$ do not affect the value of $\pi_1$ much. A common case in which an algorithm $\pi_1$'s long-term value is not affected much by the higher moments of the reward distribution is when $\pi_1$ is a sub-linear regret algorithm. In this case, for sufficiently large $T$, the average reward under $\pi_1$ will be close to that under the optimal policy. We formalize this result in Proposition \ref{prop:regretClose} below (proof in Appendix \ref{app:closeRegret}).

We define the regret of a policy $\pi_1$ in any environment $P$ as follows:
\begin{align*}
    \MC{R}_T(P, \pi_1) = \E_{P, \pi_1} \bigg[ \sum_{t=1}^T \left( R_t(a^*) - R_t \right) \bigg], \qquad \TN{where}~a^* = \TN{argmax}_{a \in \MC{A}} \, \mu_a^*.
\end{align*}
Let $\MC{P}_{\lambda^*}$ denote a class of sub-Gaussian environments with arm means $(\mu_a^*)_{a \in \MC{A}}$ and sub-Gaussian variance proxies $( \sigma_a^* )_{a \in \MC{A}}$; let $\sigma_a^* < c$ for some constant $c$ for all $a \in \MC{A}$. Note that $P, P^G \in \MC{P}_{\lambda^*}$. In Proposition \ref{prop:regretClose} below, we show that the gap $\big| \theta^{\Gn}_T - \theta_T^* \big|$ can be upper bounded as a function of the cumulative regret of $\pi_1$ divided by $T$. The second part of the proposition shows that a confidence interval for $\theta_T^G$ can be modified to form a confidence interval for $\theta_T^*$.
\begin{prop}[Value Difference Bounded by Regret]
    \label{prop:regretClose}
    For any algorithm $\pi_1$,
    \begin{align*}
        \big| \theta^{\Gn}_T - \theta_T^* \big| \leq \frac{2}{T} \sup_{P' \in \MC{P}_{\lambda^*}} \MC{R}_T(P', \pi_1).
    \end{align*}
    Furthermore, suppose we can construct an asymptotically valid confidence interval for $\theta_T^{\Gn}$ for any $\alpha \in (0, 1)$, i.e., $\lim_{\Toff \to \infty} \inf_{\lambda^* \in \Lambda} \PP_{P_{\lambda^*}, \pi_0} \left( \theta_T^{\Gn} \in \left[ \hat\theta_T \pm \nu_{\Toff}(\alpha) \right] \right) \geq 1-\alpha$, then we can construct the following asymptotically valid confidence interval for $\theta^*_T$:
    \begin{align}
        \lim_{\Toff \to \infty} \, \inf_{\lambda^* \in \Lambda} \PP_{P_{\lambda^*}, \pi_0} \bigg( \theta_T^* \in \bigg[ \hat\theta_T \pm \bigg( \nu_{\Toff}(\alpha) + \frac{2}{T} \sup_{P' \in \MC{P}_{\lambda^*}} \MC{R}_T(P', \pi_1) \bigg) \bigg] \bigg) \geq 1-\alpha.
        \label{eqn:subG-CI}
    \end{align}
\end{prop}
Note that for any bandit policy $\pi_1$ with sub-linear worst case regret under sub-Gaussian rewards, $\lim_{T \to \infty} \frac{1}{T} \cdot \sup_{P' \in \MC{P}_{\lambda^*}} \MC{R}_T(P', \pi_1) = 0$. Under the sub-Gaussian reward assumption, many common bandit algorithms like Thompson Sampling and $\epsilon$-greedy with decaying $\epsilon$ satisfy the sublinear regret condition for the worst case $P' \in \MC{P}_{\lambda^*}$~\citep{lattimore2020bandit}.

The second part of Proposition \ref{prop:regretClose} assumes that that we can construct an asymptotically valid confidence interval for $\theta_T^G$. Note that this condition can be satisfied by Theorem \ref{thm:normality} for static (or stable behavior policies) and Theorem \ref{thm:asymptoticCIAdaptive} for adaptive behavior policies (as discussed earlier in this section). 
In our experiments, we construct confidence intervals using \eqref{eqn:subG-CI} without the $\frac{2}{T} \sup_{P' \in \MC{P}_{\lambda^*}} \MC{R}_T(P', \pi_1)$ term. This is motivated by the fact that under sub-linear worst case regret algorithms $\pi_1$, we can make $\frac{1}{T} \cdot \sup_{P' \in \MC{P}_{\lambda^*}} \MC{R}_T(P', \pi_1)$ arbitrarily small for sufficiently large $T$.\footnote{Note that while we want to choose a sufficiently large $T$, we do not take the limit as $T \to \infty$ because our Proposition \ref{prop:differentiablef} (which provides sufficient conditions for the function $f_T(\lambda, \pi_1)$ to be differentiable with respect to $\lambda$) requires that $T$ is finite when $\pi_1$ is arbitrarily adaptive. We believe future work could explore whether $f_T(\lambda, \pi_1)$ is differentiable without requiring $T$ to be finite, under certain conditions on $\pi_1$.}

\bo{Unknown Variance Proxies $\big( (\sigma_a^*)^2 \big)_{a \in \MC{A}}$.}
In practice, we generally do not know the reward variance proxy parameters $\big( (\sigma_a^*)^2 \big)_{a \in \MC{A}}$. Furthermore, estimating the sub-Gaussian variance proxy is challenging, and the standard estimators for the sub-Gaussian variance proxy do not have asymptotic normality guarantees \citep{arbel2020strict}. In our experiments, we use two heuristic approaches: (i) Use an upper bound on the sub-Gaussian variance proxy, e.g., derived using the Hoeffding Lemma for bounded random variables, or (ii) Simply estimate the sub-Gaussian variance proxy $(\sigma_a^*)^2$ with the empirical variance. While (i) can be a very loose upper bound in practice, the true variance of the reward is a lower bound on the sub-Gaussian variance parameter, so (ii) will generally underestimate the true sub-Gaussian parameter. We empirically assess these two approaches in our experiments shown in Figure \ref{fig:BetaEqarmMain}.

\section{Simulation Experiments} \label{sec:simulation}
We evaluate our BSI approach in a variety of synthetic environments, as well as a semi-synthetic environment built using a randomized control trial experiment dataset \citep{gibson2022promised}. We now describe the baseline inference approaches that we compare to in all our experiments.

\bo{Baseline Inference Methods.}
Besides the \textsc{Naive \(t\)-test} that we used earlier in Figure \ref{fig:motivatingIssue} (which treats observed rewards as i.i.d. and is valid for static policies when $\pi_0=\pi_1$), we also compare BSI to four standard off-policy evaluation baselines. 
\textsc{Inverse Propensity Weighting (IPW)}~\citep{horvitz1952generalization,hirano2003efficient} and \textsc{Doubly Robust (DR)} estimator~\citep{robins1994estimation,chernozhukov2018double,jiang2016doubly} are standard importance weight-based OPE approaches that provide valid confidence intervals for data collected with static behavior policies. The final two approaches have guarantees for adaptive behavior policies and static evaluation policies. \textsc{Contextual Adaptive Doubly Robust (CADR)} baseline~\citep{bibaut2021post} is an adaptively weighted inference approach based on asymptotic approximations. The \textsc{Empirical Likelihood for Contextual Bandits (ELFCB)}~\citep{karampatziakis2020empirical} is an approach that forms confidence bounds using the empirical likelihood (rather than an asymptotic Gaussian approximation). 

Note that in many of the experiments, the behavior and evaluation policies are the same ($\pi_0 = \pi_1$), e.g., both are Thompson sampling (TS) or $\epsilon$-greedy. In this case, methods that utilize importance weights will simply have importance weights equal to $1$. Additionally, all the baseline methods we consider are primarily designed for settings in which the evaluation horizon $T$ is less than or equal to the horizon of the offline data $\Toff$. When $\Toff > T$, the baseline methods only utilize the first $T$ data points in the offline dataset $D^{(\pi_0)}$. Additional baseline details are in Appendix~\ref{app:baselines}.

\bo{Synthetic Experiments: Known Parametric Reward Distributions (Figure \ref{fig:Simulation}).}
We consider two settings each with three arms. \textit{(i) Gaussian with unknown reward variance} with arm means $(0.1, 0.2, 0.2)$ and noise variances $(1, 1, 1)$. \textit{(ii) Bernoulli rewards} with arm means $(0.35, 0.5, 0.6)$. We consider both adaptive and static behavior policies ($\pi_0=\epsilon$-Greedy, Thompson Sampling, uniform policy), and adaptive evaluation policies ($\pi_1=\epsilon$-Greedy, Thompson Sampling). For settings with adaptive $\pi_0$, we show results for both \textsc{BSI (Projection)}, which uses the confidence interval from \eqref{eqn:CIadapt}, and standard \textsc{BSI}, which replaces $\sqrt{\chi^2_{\TN{dim}(\lambda),1-\alpha}}$ in \eqref{eqn:CIadapt} with $z_{1-\alpha/2}$. Though \textsc{BSI} only has validity guarantees under non-adaptive $\pi_0$, similar to \citet{zhang2021statistical}, we find empirically that often the non-projection interval has good empirical performance in practice.
Overall, \textsc{BSI} and \textsc{BSI (Projection)} consistently achieve reliable coverage in practice, while the baseline methods often have significant undercoverage. The BSI confidence intervals are generally wider than those of the baselines, but similar in the static behavior policy case. See BSI implementation details in Appendix \ref{app:BSIdetails}.

\bo{Synthetic Experiments: Sub-Gaussian Reward Distributions (Figure \ref{fig:BetaEqarmMain}).}
We consider an environment with an unknown, sub-Gaussian reward distribution. We use a three-arm setting where the true rewards are Beta-distributed rewards with parameters $(0.35,0.65), (0.5,0.5), (0.5,0.5)$. Since the sub-Gaussian variance proxy is unknown, we consider two approaches (as described at the end of Section \ref{sec:subgaussian}). \textsc{BSI-Hoeffding} sets $(\sigma_a^*)^2 = 0.25$, the maximum variance proxy for variables on $[0,1]$ by the Hoeffding Lemma. \textsc{BSI-Empirical} estimates $(\sigma_a^*)^2$ using the empirical variance, which can underestimate the true sub-Gaussian variance proxy. Similar to the parametric reward case, we also show standard and projection versions of the BSI confidence intervals. Overall, we find that all BSI-based confidence intervals have valid coverage in practice, while many of the baselines have significant undercoverage for at least one value of $\Toff$. The \textsc{BSI-Hoeffding} intervals performed especially well in our experiments. 

\begin{figure}[b!]
    \centering
    \includegraphics[width=\textwidth,trim={1cm 0.5cm 1cm 0.5cm}]{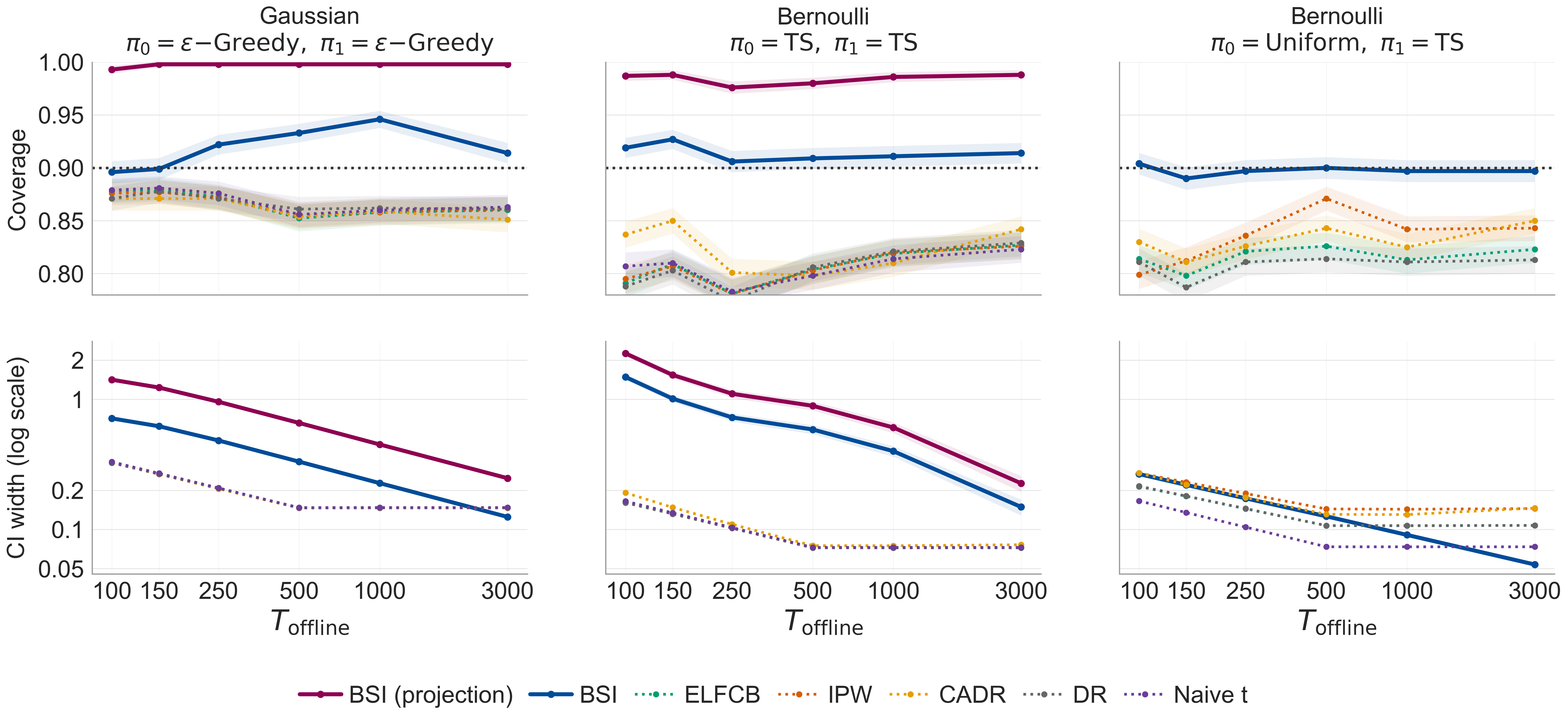}
        
    \caption{\bo{Known Parametric Reward  Experiments.} Average coverage (top row) and width (bottom row) of $90\%$ confidence intervals. All experiments use $T=500$ and are averaged across 1000 replications (shaded regions denote $1$ standard error). BSI internally uses $m=20k$ MC repetitions. Columns: \textbf{(Left)} Gaussian rewards with $3$ arms and unknown variance using $\pi_0=\epsilon\text{-Greedy}$ and $\pi_1=\epsilon\text{-Greedy}$; \textbf{(Middle)} Bernoulli rewards with $3$ arms using $\pi_0 = \pi_1 =\mathrm{Thompson~Sampling}$ (Beta-Bernoulli); \textbf{(Right)} Bernoulli rewards with $3$ arms using $\pi_0=\mathrm{Uniform}$ and $\pi_1=\mathrm{Thompson~Sampling}$ (Beta-Bernoulli).
    Note, the baseline methods only use data up to $T$, as they are designed for settings in which $T \leq \Toff$.
}
   \label{fig:Simulation}
   \vspace{-3mm}
\end{figure}

\begin{figure}[h!]
    \centering
    \includegraphics[width=0.85\textwidth,trim={1cm 0.5cm 1cm 0.5cm}]{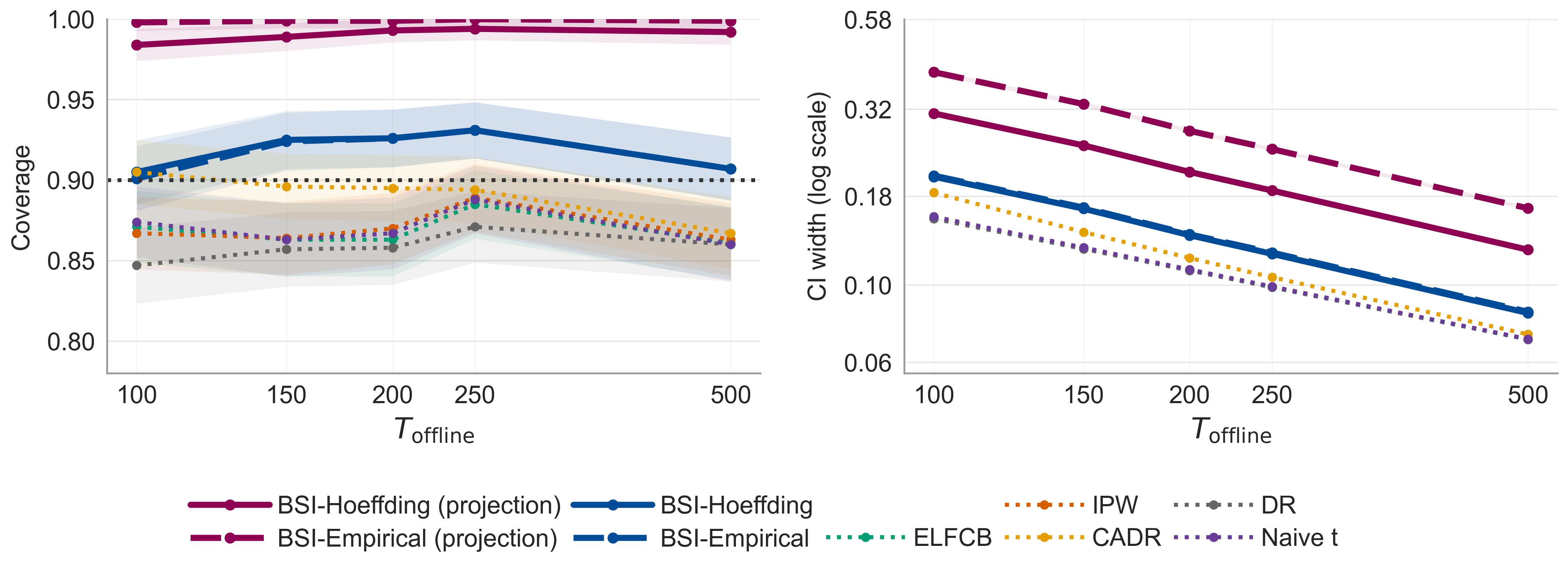}
    \caption{\bo{Sub-Gaussian Reward Distribution Experiments.} Average coverage (top row) and width (bottom row) of $90\%$ confidence intervals. All experiments use $T=500$ and are averaged across 1000 replications (shaded regions denote $1$ standard error). Three-arm setting with Beta-distributed rewards with hyper-parameters $(0.35,0.65), (0.5,0.5)$ and $ (0.5,0.5)$. Both the $\pi_0$ and $\pi_1$ are Thompson Sampling (Gaussian-Gaussian). BSI internally uses $m=20k$ MC repetitions.}
    \label{fig:BetaEqarmMain}
\end{figure}

\bo{Semi-Synthetic Data Simulation (Figure \ref{fig:semi_synthetic}).}
We build a semi-synthetic simulation environment based on a randomized controlled trial to study the effect of different financial incentives for mobile phone survey participation in Uganda \citep{gibson2022promised}; see Appendix \ref{app:semiSyntheticEnv}. The primary outcome is whether someone successfully completes the phone survey (of those who picked up the call). Phone call-based intervention experiments are increasingly using bandit algorithms to optimize intervention decisions \citep{lalan2024improving,mate2023improved,killian2023robust,dasgupta2025learning,kasy2021adaptive}. The original study by \citet{gibson2022promised} used a uniform assignment policy over three interventions. %
We use \textsc{BSI} to construct a confidence interval for the expected reward in this environment under a Thompson Sampling (Beta-Bernoulli) algorithm. We use $\Toff=T=250$ and use a uniform behavior policy to collect $D^{(\pi_0)}$. As seen in Figure \ref{fig:semi_synthetic}, we find that \textsc{BSI} has accurate coverage and confidence intervals that are narrower than some baselines, which all have significant undercoverage.

\begin{figure}[h]
    \centering
    \includegraphics[width=\textwidth,trim={1cm 0.5cm 1cm 0.5cm}]{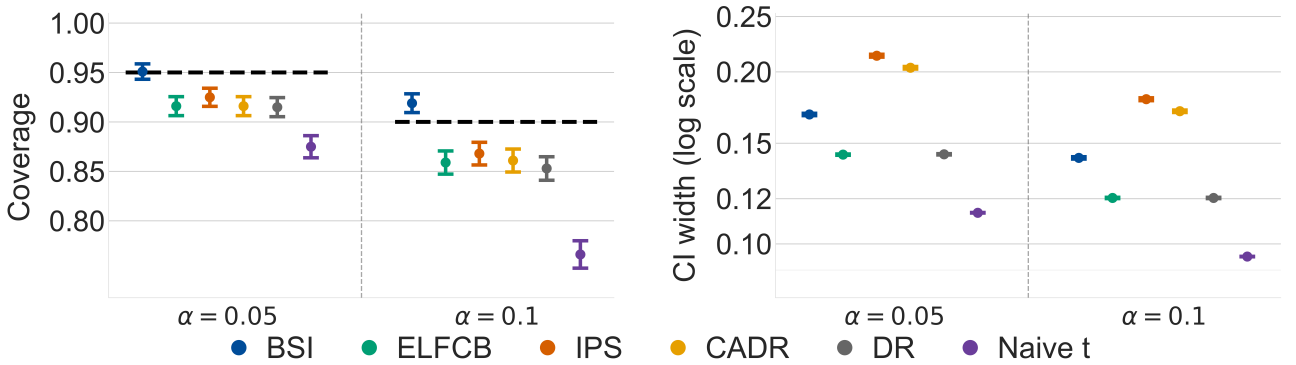}  
    \caption{\bo{Semi-Synthetic Experiments.} The semi-synthetic simulator based on a phone survey incentive randomized control trial  \citep{gibson2022promised} has $3$ arms and binary rewards. We use a uniform behavior policy and TS evaluation policy, and use $\Toff=T=250$. We average results across $1000$ replications (error bars denote $1$ standard error). BSI internally uses $m=20k$ MC repetitions.}
    \label{fig:semi_synthetic}
    \vspace{-3mm}
\end{figure}

 \section{Discussion}

We believe that compute-heavy, simulation-based methods for providing frequentist statistical inference guarantees are relatively underexplored in the context of adaptively collected data. Intuitively, evaluating an algorithm by repeatedly simulating it is very natural. Previous works have proposed using such simulation-based approaches to assess bandit algorithms \citep{ghosh2024did,cho2025simulation}, but we believe we are the first work that leverages adaptively collected data to build a simulator for statistical inference that has both rigorous theoretical guarantees and is widely applicable to a large range of adaptive algorithms.

In this work, we present an approach for constructing valid confidence intervals for the value of adaptive policies. Since our approach involves using data to fit a simulator, we can utilize on-policy or off-policy data and provide inference guarantees for horizons not observed in the original dataset. Furthermore, we show that our approach can extend beyond parametric reward setting to general sub-Gaussian rewards when the evaluation policy is regret-minimizing. Although this work only studies the multi-arm bandit setting, we believe this approach could potentially be extended to more complex settings like contextual bandits and Markov decision processes.

\bibliographystyle{plainnat}
\bibliography{main}

\clearpage
\appendix

\addtocontents{toc}{\protect\setcounter{tocdepth}{2}}
\tableofcontents

\clearpage

  \newpage
  \section{Theory Results} \label{app:proof}

\subsection{Proof of Theorem \ref{thm:normality}: Bandit Simulation for Inference with Static $\pi_0$}
\label{app:proofNormality}

\begin{lemma}[Uniform Linearization]
    \label{lemma:linearization}
    Suppose $\sqrt{\Toff} \big( \hat{\lambda} - \lambda^* \big) \Dto \N \big(0, \Sigma_{\lambda^*,\pi_0} \big)$ uniformly in distribution over $P_{\lambda^*,\pi_0} \in \{ P_{\lambda,\pi_0} : \lambda \in \Lambda \}$. Also, let $\sup_{\lambda^* \in \Lambda} \| \Sigma_{\lambda^*} \|_2 < \infty$. Let $f$ be a uniformly differentiable function with respect to $\lambda \in \Lambda$, satisfying \eqref{eqn:uniformlyDiff}. Then, 
    \begin{align}
        \sqrt{\Toff} \big( f_T( \hat{\lambda}, \pi_1) - f_T( \lambda^*, \pi_1) \big) - \nabla f_T(\lambda^*, \pi_1) \sqrt{\Toff} \big( \hat{\lambda} - \lambda^* \big) \Pto 0
        \quad \TN{as} \quad \Toff \to \infty,
        \label{eqn:uniformLinearization}
    \end{align}
    uniformly over $P_{\lambda^*,\pi_0} \in \{ P_{\lambda,\pi_0} : \lambda \in \Lambda \}$.
\end{lemma}

\begin{proof}
It is useful to define the following normalized remainder term:
\begin{align*}
    \nu(\lambda^*, h) \triangleq \frac{ \| f_T( \lambda^* + h, \pi_1) - f_T( \lambda^*, \pi_1) - \nabla h^\top f_T( \lambda^*, \pi_1) \|_2 }{ \| h \|_2}.
\end{align*}
Note, that we can bound the  norm of \eqref{eqn:uniformLinearization} as follows:
\begin{align*}
    \sqrt{\Toff} \cdot \big\| f_T( \hat{\lambda}, \pi_1 ) - f_T( \lambda^*, \pi_1 ) - \nabla f_T(\lambda^*, \pi_1) \big( \hat{\lambda} - \lambda^* \big) \big\|_2 
    \leq \sqrt{\Toff} \cdot \big\| \hat{\lambda} - \lambda^* \big\|_2 \cdot \nu \big( \lambda^*, \hat{\lambda} - \lambda^* \big)
\end{align*}
Thus, it is sufficient to show the following converges to zero for any $\epsilon > 0$:
\begin{align}
    \sup_{\lambda^* \in \Lambda} \PP_{\lambda^*,\pi_0} \left( \sqrt{\Toff} \cdot \| \hat{\lambda} - \lambda^* \|_2 \cdot \nu \big( \lambda^*, \hat{\lambda} - \lambda^* \big) > \epsilon \right)
    \label{eqn:linearSup}
\end{align}
Note that since $\sqrt{\Toff} ( \hat{\lambda} - \lambda^* )$ is uniformly tight by assumption, there exists some constant $c < \infty$ such that $\sup_{\lambda^* \in \Lambda} \PP_{\lambda^*}  \big( \sqrt{\Toff} \cdot \big\| \hat{\lambda} - \lambda^* \big\|_2 > c \big) \to 0$. Thus, we can upper bound \eqref{eqn:linearSup} by 
\begin{align*}
    \leq \sup_{\lambda^* \in \Lambda} \PP_{\lambda^*,\pi_0} \big( \sqrt{\Toff} \cdot \big\| \hat{\lambda} - \lambda^* \big\|_2 > c \big) +  \sup_{\lambda^* \in \Lambda} \PP_{\lambda^*,\pi_0} \big( c \cdot \nu \big( \lambda^*, \hat{\lambda} - \lambda^* \big) > \epsilon \big)  \to 0.
\end{align*}
Above, the second term goes to zero by our uniform differentiability assumption on $f$, and since $\hat{\lambda} - \lambda^* \Pto 0$ %
uniformly over $P_{\lambda^*,\pi_0} \in \{ P_{\lambda,\pi_0} : \lambda \in \Lambda \}$.
\end{proof}

\begin{customthm}{\ref{thm:normality}}[Bandit Simulation for Inference: Static $\pi_0$]
    Under Assumptions \ref{assump:lambdaNormality}-\ref{assump:differentiable}, $\hat{\theta}_T$ satisfies
    \begin{align*}
        \sqrt{ \Toff } (\hat{\theta}_T - \theta_T^*) \Dto \N \Big( 0, \, \nabla_\lambda f_T(\lambda^*, \pi_1)^\top \Sigma_{\lambda^*,\pi_0} \nabla_\lambda f_T(\lambda^*, \pi_1) \Big) \quad \TN{as} \quad \Toff \to \infty,
    \end{align*}
    uniformly over the underlying  $P_{\lambda^*, \pi_0} \in \{ P_{\lambda, \pi_0} : \lambda \in \Lambda \}$.
    Thus, assuming a non-degenerate limiting variance, i.e., $\inf_{\lambda^* \in \Lambda} \nabla_\lambda f_T(\lambda^*, \pi_1)^\top \Sigma_{\lambda^*, \pi_0} \nabla_\lambda f_T(\lambda^*, \pi_1) > 0$, then for any %
    $\alpha \in (0, 1)$, %
    \begin{align*}
        \lim_{\Toff \to \infty} \inf_{\lambda^* \in \Lambda} \PP_{P_{\lambda^*}, \pi_0} \bigg( \theta_T^* \in \bigg[ \hat{\theta}_T \pm \frac{ z_{1-\alpha/2} }{ \sqrt{ \Toff} } \sqrt{ \nabla_\lambda f_T(\hat{\lambda}, \pi_1)^\top \hat\Sigma_{\hat{\lambda},\pi_0} \nabla_\lambda f_T(\hat{\lambda}, \pi_1)  } 
        \bigg] \bigg) \geq 1-\alpha,
    \end{align*}
    where $\hat\Sigma_{\hat{\lambda},\pi_0}$ is  any  positive definite matrix satisfying $\hat\Sigma_{\hat{\lambda},\pi_0} \Pto \Sigma_{\lambda^*, \pi_0}$ uniformly over $\lambda^\ast \in \Lambda$.   
\end{customthm}

\begin{proof}
By Assumption \ref{assump:lambdaNormality}, $\sqrt{ \Toff } ( \hat{\lambda} - \lambda^* ) \Dto N \big(0, \Sigma_{\lambda^*,\pi_0} \big)$ uniformly over $P_{\lambda^*} \in \{ P_{\lambda} : \lambda \in \Lambda \}$. By Slutsky's Theorem,
\begin{align}
    \nabla f_T(\lambda^*, \pi_1) \sqrt{ \Toff } \big( \hat{\lambda} - \lambda^* \big) \Dto \N \big(0, \nabla f_T(\lambda^*, \pi_1)^\top \Sigma_{\lambda^*,\pi_0} \nabla f_T(\lambda^*, \pi_1) \big)
    \label{eqn:intermed0}
\end{align}
uniformly over $P_{\lambda^*, \pi_0} \in \{ P_{\lambda, \pi_0} : \lambda \in \Lambda \}$. By Lemma \ref{lemma:linearization},
\begin{align}
    \sqrt{ \Toff } \big( f_T( \hat{\lambda}, \pi_1 ) - f_T( \lambda^*, \pi_1) \big) - \nabla f_T(\lambda^*, \pi_1) \sqrt{ \Toff } \big( \hat{\lambda} - \lambda^* \big) \Pto 0,
    \label{eqn:intermed1}
\end{align}
uniformly over $P_{\lambda^*, \pi_0} \in \{ P_{\lambda, \pi_0} : \lambda \in \Lambda \}$. The asymptotic normality result in the Theorem holds by \eqref{eqn:intermed0}, \eqref{eqn:intermed1}, and Slutsky's theorem. 

Thus, we have that
\begin{align*}
    \left( \nabla_\lambda f_T(\lambda^*, \pi_1)^\top \Sigma_{\lambda^*,\pi_0}  \nabla_\lambda f_T(\lambda^*, \pi_1) \right)^{-1/2}
    \sqrt{ \Toff } (\hat{\theta}_T - \theta_T^*) \Dto \N \left( 0, 1 \right),
\end{align*}
uniformly over $P_{\lambda^*, \pi_0} \in \{ P_{\lambda, \pi_0} : \lambda \in \Lambda \}$.
For the confidence interval result, by Slutsky's Theorem it is sufficient to show that
\begin{align*}
    \frac{
    \nabla_\lambda f_T(\hat\lambda, \pi_1)^\top \hat\Sigma_{\hat\lambda,\pi_0}  \nabla_\lambda f_T(\hat\lambda, \pi_1) }{
    \nabla_\lambda f_T(\lambda^*, \pi_1)^\top \Sigma_{\lambda^*,\pi_0} \nabla_\lambda f_T(\lambda^*, \pi_1) 
    }
    \Pto 1,
\end{align*}
uniformly over $P_{\lambda^*, \pi_0} \in \{ P_{\lambda, \pi_0} : \lambda \in \Lambda \}$. Since the denominator $\nabla_\lambda f_T(\lambda^*, \pi_1)^\top \Sigma_{\lambda^*,\pi_0}  \nabla_\lambda f_T(\lambda^*, \pi_1)$ is uniformly bounded above zero by our non-degeneracy assumption, \\ $\inf_{\lambda^* \in \Lambda} \nabla_\lambda f_T(\lambda^*, \pi_1)^\top \Sigma_{\lambda^*, \pi_0} \nabla_\lambda f_T(\lambda^*, \pi_1) > 0$, it sufficient to show that 
\begin{align}
    \label{eqn:consistentVar}
    \nabla_\lambda f_T(\hat\lambda, \pi_1)^\top \hat\Sigma_{\hat\lambda,\pi_0}  \nabla_\lambda f_T(\hat\lambda, \pi_1) \Pto \nabla_\lambda f_T(\lambda^*, \pi_1)^\top \Sigma_{\lambda^*,\pi_0} \nabla_\lambda f_T(\lambda^*, \pi_1),
\end{align}
uniformly over $P_{\lambda^*, \pi_0} \in \{ P_{\lambda, \pi_0} : \lambda \in \Lambda \}$. \eqref{eqn:consistentVar} above holds by Slutsky's Theorem since $\hat\Sigma_{\hat\lambda,\pi_0}$ is uniformly consistent, and since $\nabla_\lambda f_T(\hat\lambda, \pi_1) \Pto \nabla_\lambda f_T(\lambda^*, \pi_1)$ uniformly (which holds by continuous mapping theorem since $\hat\lambda \Pto \lambda^*$ uniformly and $\nabla f_T(\lambda, \pi_1)$ is continuous by Assumption \ref{assump:differentiable}).
\end{proof}

\subsection{Proof of Proposition \ref{prop:differentiablef}: Derivative of $f_T$}
\label{app:differentiability}

\begin{customprop}{\ref{prop:differentiablef}}[Derivative of $f_T$]
    $f_T(\lambda, \pi_1)$ satisfies Assumption \ref{assump:differentiable} for an open set $\Lambda$ with a compact closure, if for any $\lambda = (\lambda_a)_{a \in \MC{A}} \in \Lambda$ the following conditions are satisfied:
    \begin{itemize}[leftmargin=*]
        \item \textnormal{\bo{Finite second moment.}} $\E_{\lambda_a} \big[ R_t(a)^2 \big] < \infty$ and $\E_{\lambda_a} \big[ \left\| \nabla_{\lambda_a} \log P_{\lambda_a} \left( R_t(a) \right) \right\|_1^2 \big] < \infty$
        \item \textnormal{\bo{Continuous.}} Let $P_{\lambda_a} \left( R_t(a) \right)$ and $\nabla_{\lambda_a} \log P_{\lambda_a} \left( R_t(a) \right)$ be almost surely continuous in $\lambda_a$.
    \end{itemize}
    Moreover, the derivative $\nabla_{\lambda} f_T(\lambda, \pi_1)$ can be written as follows:
    \begin{align}
        \label{eqn:fgrad}
        \nabla_{\lambda} f_T(\lambda, \pi_1) 
        = \frac{1}{T} \sum_{t=1}^{T} \E_{\lambda, \pi_1} \bigg[ \nabla_{\lambda} \log P_{\lambda_a} (R_{t} \mid A_{t}) \cdot \sum_{t'=t}^T  R_{t'} \bigg].
    \end{align}
\end{customprop}

\begin{remark}[Gradients for static $\pi_1$]
Note, in the special case that $\pi_1$ is a static policy, $\nabla_\lambda f_T(\lambda, \pi_1) = \frac{1}{T} \sum_{t=1}^{T} \E_{\lambda, \pi_1} \left[ \nabla_\lambda \log P_{\lambda_a} (R_{t} \mid A_{t}) \cdot R_{t'} \right]$. This is because for any $t \neq t'$, $R_t$ and $R_{t'}$ are independent, and we have  $\E_{\lambda, \pi_1} \left[ R_t \cdot \nabla_\lambda \log P_{\lambda_a} (R_{t'} \mid A_{t'}) \mid \HH_{t'-1} \right] 
= R_t \cdot \E_{\lambda, \pi_1} \left[ \nabla_\lambda \log P_{\lambda_a} (R_{t'} \mid A_{t'}) \mid \HH_{t'-1} \right] = 0$. Thus, by the law of iterated expectations,
\begin{align*}
    \frac{1}{T} \sum_{t=1}^{T} \E_{\lambda, \pi_1} \bigg[ \nabla_\lambda \log P_{\lambda_a} (R_{t} \mid A_{t}) \cdot \sum_{t'=t}^T  R_{t'} \bigg] 
    = \frac{1}{T} \sum_{t=1}^{T} \E_{\lambda, \pi_1} \big[ \nabla_\lambda \log P_{\lambda_a} (R_{t} \mid A_{t}) \cdot R_t \big]
\end{align*}
\end{remark}

\begin{proof}
We use $\PP_{\lambda, \pi_1}$ to denote the distribution of the trajectory induced by policy $\pi_1$ when applied on the environment parameterized by $\lambda$. Note that the likelihood of any sequence of actions and rewards $h_{T} \triangleq \{ (a_t, r_t) \}_{t=1}^T$ under $\PP_{\lambda, \pi_1}$ can be written as follows:
\begin{align}
    \PP_{\lambda, \pi_1}( A_{1:T} = a_{1:T}, R_{1:T} = r_{1:T} ) 
    = \prod_{t=1}^T \big\lbrace \pi_1(a_t \mid h_{t-1}) \cdot P_{\lambda_a} (r_t \mid a_t) \big\rbrace.
    \label{eqn:likelihood}
\end{align}
For the proof, it is useful to simplify the following derivative:
\begin{align}
    &\nabla_{\lambda} \PP_{\lambda, \pi_1}( A_{1:T} = a_{1:T}, R_{1:T} = r_{1:T} )
    = \prod_{t=1}^T \pi_1(a_t \mid h_{t-1}) \cdot \nabla_\lambda \bigg\{ \prod_{t=1}^T P_{\lambda_{a_t}} (r_t \mid a_t) \bigg\} \nonumber \\
    &= \bigg\{ \prod_{t=1}^T \pi_1(a_t \mid h_{t-1}) 
    \bigg\} \cdot \bigg\{ \prod_{t=1}^T P_{\lambda_{a_t}} (r_t \mid a_t ) \bigg\} \cdot \nabla_{\lambda} \log \bigg\{ \prod_{t=1}^T P_{\lambda_{a_t}} (r_t \mid a_t) \bigg\} \nonumber  \\
    &= \prod_{t=1}^T \big\{ \pi_1(a_t \mid h_{t-1}) \cdot P_{\lambda_{a_t}} (r_t \mid a_t ) \big\} \cdot \sum_{t'=1}^T \nabla_{\lambda} \log P_{\lambda_{a_{t'}}} (r_{t'} \mid a_{t'} ) 
    \label{eqn:rewriteDeriv}
\end{align}

\paragraph{Deriving the derivative.}
By the definition of $f_T$ and the law of total probability, $\nabla_{\lambda} f_T(\lambda, \pi_1)$
\begin{align*}
    = \nabla_{\lambda} \E_{\lambda,\pi_1} [ \bar{R}_T ] 
    = \nabla_{\lambda} \bigg\{ \sum_{a_{1:T} \in \mathcal{A}^T} \int_{r_{1:T} \in \real^T} \bar{r}_T \cdot \PP_{\lambda, \pi_1}( A_{1:T} = a_{1:T}, R_{1:T} = r_{1:T} ) d r_{1:T} \bigg\},
\end{align*}
for $\bar{R}_T \triangleq \frac{1}{T} \sum_{t=1}^T R_t$ and $\bar{r}_T \triangleq \frac{1}{T} \sum_{t=1}^T r_t$.

Next, we swap the order of the derivative and integral above using the dominated convergence theorem. 
\begin{align}
    &= \sum_{a_{1:T} \in \mathcal{A}^T} \int_{r_{1:T} \in \real^T} \nabla_{\lambda} \big\{ \bar{r}_T \cdot \PP_{\lambda, \pi_1}( A_{1:T} = a_{1:T}, R_{1:T} = r_{1:T} ) \big\} d r_{1:T}
    \label{eqn:bigswap}
\end{align}
Specifically, we can apply the dominated convergence theorem theorem because
\begin{itemize}[leftmargin=*]
    \item Note $|\bar{R}_T| = \big|\frac{1}{T} \sum_{t=1}^T R_t \big|\leq \frac{1}{T} \sum_{t=1}^T \left| R_t \right| \leq \frac{1}{T} \sum_{t=1}^T \sum_{a \in \MC{A}} \left| R_t(a) \right|$ a.s.
    \item By \eqref{eqn:rewriteDeriv},
    \begin{align*}
        \left| \bar{R}_T \cdot \nabla_{\lambda} \PP_{\lambda, \pi_1}( A_{1:T}, R_{1:T} ) \right|
        &\leq 1 \cdot \Bigg| \bar{R}_T \cdot \sum_{t'=1}^T \nabla_{\lambda} \log P_{\lambda_{A_{t'}}} (R_{t'} \mid A_{t'} ) \Bigg| \\
        &\leq |\bar{R}_T| \cdot \sum_{t'=1}^T \left| \nabla_{\lambda} \log P_{\lambda_{A_{t'}}} (R_{t'} \mid A_{t'} ) \right| \\
        &\leq \sum_{a \in \MC{A}} 
        \left\{ \frac{1}{T} \sum_{t=1}^T \left| R_t(a) \right| \right\rbrace  \cdot \left\lbrace \sum_{t'=1}^T \big| \nabla_{\lambda} \log P_{\lambda_{a}} (R_{t'}(a) ) \big| \right\}
    \end{align*}
    The above is integrable since $P_{\lambda_a}( R_t(a) )$ and $\nabla_{\lambda} \log P_{\lambda_{A_{t'}}} ( R_t(a) )$ have finite second moment by assumption of the proposition.
\end{itemize}

By \eqref{eqn:rewriteDeriv}, we have that \eqref{eqn:bigswap} equals the following:
\begin{align*}
    &= \sum_{a_{1:T} \in \mathcal{A}^T} \int_{r_{1:T} \in \real^T} \bar{r}_T \cdot \bigg( \sum_{t'=1}^T \nabla_{\lambda} \log P_{\lambda_{a_{t'}}} (r_{t'} \mid a_{t'}) \bigg) \cdot \PP_{\lambda, \pi_1}( A_{1:T} = a_{1:t}, R_{1:T} = r_{1:T} ) d r_{1:T}
\end{align*}
By the law of total probability, %
\begin{align*}
    &= \E_{\lambda, \pi_1} \bigg[ \bar{R}_T \cdot \bigg( \sum_{t'=1}^T \nabla_{\lambda} \log P_{\lambda_{A_{t'}}} (R_{t'} \mid A_{t'}) \bigg) \bigg]
    = \frac{1}{T} \sum_{t=1}^T \E_{\lambda, \pi_1} \bigg[ R_t \cdot \bigg( \sum_{t'=1}^T \nabla_{\lambda} \log P_{\lambda_{A_{t'}}} (R_{t'} \mid A_{t'}) \bigg) \bigg]
\end{align*}
By law of iterated expectations, since for $t < t'$, $\E_{\lambda, \pi_1} \big[ R_t \cdot \nabla_{\lambda} \log P_{\lambda_{A_{t'}}} (R_{t'} \mid A_{t'}) \mid \HH_{t'-1} \big] \\
= R_t \cdot \E_{\lambda, \pi_1} \big[ \nabla_\lambda \log P_{\lambda_{A_{t'}}} (R_{t'} \mid A_{t'}) \mid \HH_{t'-1} \big] 
= R_t \cdot \int_{\real} \nabla_\lambda \log P_{\lambda_{A_{t'}}} (r \mid A_{t'}) \cdot P_{\lambda_{A_{t'}}} (r \mid A_{t'}) dr \\
= R_t \cdot \int_{\real} \nabla_\lambda P_{\lambda_{A_{t'}}} (r \mid A_{t'}) dr
= R_t \cdot \nabla_\lambda \int_{\real} P_{\lambda_{A_{t'}}} (r \mid A_{t'}) dr
= R_t \cdot \nabla_\lambda (1) = 0$, 
\begin{align*}
    = \frac{1}{T} \sum_{t=1}^T \E_{\lambda, \pi_1} \bigg[  R_t \cdot \bigg( \sum_{t'=1}^{t} \nabla_{\lambda} \log P_{\lambda_{A_{t'}}} (R_{t'} \mid A_{t'}) \bigg) \bigg]
    = \frac{1}{T} \sum_{t=1}^T \sum_{t'=1}^{t} \E_{\lambda, \pi_1} \left[  R_t \cdot \nabla_{\lambda} \log P_{\lambda_{A_{t'}}} (R_{t'} \mid A_{t'}) \right]
\end{align*}
By swapping the summation order,
\begin{align*}
    = \frac{1}{T} \sum_{t'=1}^{T} \sum_{t=t'}^T \E_{\lambda, \pi_1} \left[  R_t \cdot \nabla_{\lambda} \log P_{\lambda_{A_{t'}}} (R_{t'} \mid A_{t'}) \right]
    = \frac{1}{T} \sum_{t=1}^{T} \E_{\lambda, \pi_1} \bigg[ \nabla_{\lambda} \log P_{\lambda_{A_t}} (R_{t} \mid A_{t}) \cdot \sum_{t'=t}^T  R_{t'} \bigg].
\end{align*}
The last equality swaps $t$ and $t'$.

\paragraph{Uniformly differentiable.}
Let $\bar{\Lambda}$ be the closure of $\Lambda$. Since $\bar{\Lambda}$ is compact by assumption, by the Heine-Cantor Theorem, $f_T(\lambda, \pi_1)$ is uniformly differentiable if $\nabla_{\lambda} f_T(\lambda, \pi_1)$ is continuous at all $\lambda \in \bar{\Lambda}$.
\begin{align*}
    &\frac{1}{T} \sum_{t=1}^{T} \E_{\lambda, \pi_1} \bigg[ \nabla_\lambda \log P_{\lambda_{A_t}} (R_{t} \mid A_{t}) \cdot \sum_{t'=t}^T  R_{t'} \bigg] \\
    &= \frac{1}{T} \sum_{t=1}^{T} \sum_{a_{1:T} \in \mathcal{A}^T} \int_{r_{1:T} \in \real^T} \hspace{-2mm} \nabla_\lambda \log P_{\lambda_{a_t}} (r_t \mid a_t) \cdot \bigg( \sum_{t'=t}^T r_{t'} \bigg) \cdot \prod_{t'=1}^T \big\{ \pi_1(a_{t'} \mid h_{t'-1}) P_{\lambda_{a_{t'}}} (r_{t'} \mid a_{t'}) \big\} d r_{1:T}
\end{align*}
Since $\nabla_{\lambda_a} \log P_{\lambda_a} (R_t \mid A_t = a)$ and $P_{\lambda_{a}} (R_t \mid A_t = a)$ are almost surely continuous in $\lambda_a$ for every $a \in \MC{A}$, for any sequence $\lambda^{(n)} \to \lambda$ as $n \to \infty$,
\begin{multline*}
    \nabla_{\lambda^{(n)}} \log P_{\lambda_{A_t}^{(n)}} (R_t \mid A_t) \cdot \bigg( \sum_{t'=t}^T R_{t'} \bigg) \cdot \prod_{t'=1}^T \big\{ \pi_1(A_{t'} \mid H_{t'-1}) \cdot P_{\lambda_{A_{t'}}^{(n)}} (R_{t'} \mid A_{t'} ) \big\} \\
    \to \nabla_{\lambda} \log P_{\lambda_{A_t}} (R_t \mid A_t) \cdot \bigg( \sum_{t'=t}^T R_{t'} \bigg) \cdot \prod_{t'=1}^T \Big\{ \pi_1(A_{t'} \mid H_{t'-1}) \cdot P_{\lambda_{A_{t'}}} (R_{t'} \mid A_{t'} ) \Big\} ~~ \TN{w.p.}~1.
\end{multline*}
Furthermore, by Dominated Convergence Theorem 
\begin{multline*}
    \E_{\lambda^{(n)}, \pi_1} \bigg[ \nabla_{\lambda^{(n)}} \log P_{\lambda_{A_t}^{(n)}} (R_t \mid A_t) \cdot \bigg( \sum_{t'=t}^T R_{t'} \bigg) \cdot \prod_{t'=1}^T \Big\{ \pi_1(A_{t'} \mid H_{t'-1}) \cdot P_{\lambda_{A_{t'}}^{(n)}} (R_{t'} \mid A_{t'} ) \Big\} \bigg] \\
    \to \E_{\lambda, \pi_1} \bigg[ \nabla_{\lambda} \log P_{\lambda_{A_t}} (R_t \mid A_t) \cdot \bigg( \sum_{t'=t}^T R_{t'} \bigg) \cdot \prod_{t'=1}^T \Big\{ \pi_1(A_{t'} \mid H_{t'-1}) \cdot P_{\lambda_{A_{t'}}} (R_{t'} \mid A_{t'} ) \Big\} \bigg]
\end{multline*}
We are able to apply dominated convergence theorem because the following holds with probability $1$:
\begin{multline*}
    \bigg| \nabla_{\lambda^{(n)}} \log P_{\lambda_{A_t}^{(n)}} (R_t \mid A_t) \cdot \bigg( \sum_{t'=t}^T R_{t'} \bigg) \cdot \prod_{t'=1}^T \Big\{ \pi_1(A_{t'} \mid H_{t'-1}) \cdot P_{\lambda_{A_{t'}}^{(n)}} (R_{t'} \mid A_{t'} ) \Big\} \bigg| \\
    \leq \sum_{a \in \MC{A}} \left| \nabla_{\lambda^{(n)}} \log P_{\lambda_a^{(n)}} \big( R_t(a) \big) \right| \cdot \sum_{t'=t}^T \sum_{a \in \MC{A}}\left| R_{t'}(a) \right|,
\end{multline*}
and $\left| R_{t'}(a) \right|$ and $\big| \nabla_{\lambda^{(n)}} \log P_{\lambda_a^{(n)}} \big( R_t(a) \big) \big|$ have finite second moment by assumption.
\end{proof}

\subsection{Proof of Theorem \ref{thm:asymptoticCIAdaptive}: Simulation-Based Inference: Adaptive $\pi_0$}
\label{app:proofCIAdaptive}

\begin{customthm}{\ref{thm:asymptoticCIAdaptive}}[Simulation-Based Inference Confidence Intervals: Adaptive $\pi_0$]
    Under Assumptions \ref{assump:differentiable}-\ref{assump:lambdaNormalityAdaptive}, we can construct the following asymptotically valid confidence intervals for $\theta^*$ with any $\alpha \in (0, 1)$,
    \begin{align*}
        \lim_{\Toff \to \infty} \inf_{\lambda^* \in \Lambda} \PP_{P_{\lambda^*}, \pi_0} \Big( \theta_T^* \in \Big[ \hat{\theta}_T \pm \sqrt{ \chi_{\TN{dim}(\lambda), 1-\alpha}^2 } \cdot \big\| (V_{\Toff}^{\top} V_{\Toff})^{-1/2} \nabla f_T( \hat\lambda, \pi_1 ) \big\|_2 \Big] \Big) \geq 1-\alpha.
    \end{align*}
\end{customthm}
\begin{proof}
By Assumption \ref{assump:lambdaNormalityAdaptive}, 
\begin{align*}
    \lim_{T_{\mathrm{offline}} \to \infty} \inf_{\lambda^* \in \Lambda} \PP_{\lambda^*, \pi_0} \left( ( \hat{\lambda} - \lambda^* )^\top V_{\Toff}^{\top} V_{\Toff} ( \hat{\lambda} - \lambda^* ) \leq \chi^2_{\TN{dim}(\lambda), 1-\alpha} \right) \geq 1-\alpha.
\end{align*}
Thus, we have that, $\lim_{T_{\mathrm{offline}} \to \infty} \inf_{\lambda^* \in \Lambda} \PP_{\lambda^*, \pi_0} \left( \lambda^* \in C_{T_{\mathrm{offline}}} \right) \geq 1-\alpha$ for
\begin{align*}
    C_{T_{\mathrm{offline}}} \triangleq \Big\{ \lambda : \big( \hat{\lambda} - \lambda \big)^\top V_{\Toff}^{\top} V_{\Toff} \big( \hat{\lambda} - \lambda \big) \leq \chi^2_{\TN{dim}(\lambda), 1-\alpha} \Big\}.
\end{align*}
This means that
\begin{align}
    \lim_{T_{\mathrm{offline}} \to \infty} \inf_{\lambda^* \in \Lambda} \PP_{\lambda^*, \pi_0} \big( f_T(\lambda^*, \pi_1) \in \big\{ f_T(\lambda, \pi_1) : \lambda \in C_{T_{\mathrm{offline}}} \big\} \big) \geq 1-\alpha.
    \label{eqn:appDeltaAdaptive}
\end{align}

\paragraph{Simplifying the confidence interval in \eqref{eqn:appDeltaAdaptive}.}
By reparameterizing $\lambda = \hat\lambda + h$, note that $\lambda \in C_{\Toff}$ if and only if $h^\top V_{\Toff}^{\top} V_{\Toff} h \leq \chi^2_{\TN{dim}(\lambda), 1-\alpha}$. Thus,
\begin{align}
    C_{T_{\mathrm{offline}}} = \left\{ \hat\lambda + h :  h^\top V_{\Toff}^{\top} V_{\Toff} h \leq \chi^2_{\TN{dim}(\lambda), 1-\alpha} \right\}.
    \label{eqn:adapCrewrite}
\end{align}

Since, by Assumption \ref{assump:lambdaNormalityAdaptive}, the minimum eigenvalue of $V_{\Toff}$ goes to infinity with probability $1$; thus, with probability going to $1$, there must exist some $\delta < \infty$, such that for all $h$ such that $h^\top V_{\Toff}^{\top} V_{\Toff} h \leq \chi^2_{\TN{dim}(\lambda), 1-\alpha}$, we have $\| h \|_2 < \delta$.
Note also by the uniform differentiability of $f$ (Assumption \ref{assump:differentiable}),
\begin{align*}
    \sup_{\lambda^* \in \Lambda} \, \sup_{0 < \| h \|_2 < \delta} \, \left| f \big( \hat\lambda + h, \, \pi_1 \big) - f_T( \hat\lambda, \pi_1 ) - \nabla f_T( \hat\lambda, \pi_1 )^\top h \right| \Pto 0.
\end{align*}
Thus, we have that
\begin{align*}
    \big\{ f_T(\lambda, &\pi_1) : \lambda \in C_{T_{\mathrm{offline}}} \big\} = f_T( \hat\lambda, \pi_1 ) + \left\{ \nabla f_T( \hat\lambda, \pi_1 )^\top h : h^\top V_{\Toff}^{\top} V_{\Toff} h \leq \chi^2_{\TN{dim}(\lambda), 1-\alpha} \right\} + \sup_{\lambda^* \in \Lambda} o_{P_{\lambda^*, \pi_0}} (1).
\end{align*}
Note that by Exercise 4.21 in \citet{boyd2004convex},
\begin{align*}
    \sup_{h \, : \, h^\top V_{\Toff}^{\top}  V_{\Toff} h \leq \chi^2_{\TN{dim}(\lambda), 1-\alpha}} \nabla f_T( \hat\lambda, \pi_1 )^\top h &= \sqrt{ \chi^2_{\TN{dim}(\lambda), 1-\alpha} \nabla f_T( \hat\lambda, \pi_1 )^\top (V_{\Toff}^{\top} V_{\Toff})^{-1} \nabla f_T( \hat\lambda, \pi_1 ) } \\
    &= \sqrt{ \chi^2_{\TN{dim}(\lambda), 1-\alpha} } \cdot \big\| (V_{\Toff}^{\top} V_{\Toff})^{-1/2} \nabla f_T( \hat\lambda, \pi_1 ) \big\|_2 
\end{align*}
Thus,
\begin{align*}
    \big\{ f_T(\lambda, \pi_1) : \lambda \in C_{T_{\mathrm{offline}}} \big\} 
    = \Big[ f_T(\hat\lambda, \pi_1) \pm \sqrt{ \chi^2_{\TN{dim}(\lambda), 1-\alpha} }  \cdot \big\| (V_{\Toff}^{\top} V_{\Toff})^{-1/2} \nabla f_T( \hat\lambda, \pi_1 ) \big\|_2 \Big] + \sup_{\lambda^* \in \Lambda} o_{P_{\lambda^*, \pi_0}} (1).
\end{align*}
\end{proof}

\subsection{Proof of Corollary \ref{corr:contrast}: Policy Contrasts}
\label{app:policyContrast}

\begin{customcor}{\ref{corr:contrast}}[Policy Contrasts]
    Under Assumptions \ref{assump:lambdaNormality} and~\ref{assump:differentiable}, the plug-in estimator $\hat{\delta}_T$ %
    satisfies  
    \begin{align*}
        \sqrt{ \Toff } (\hat{\delta}_T - \delta_T^*) \Dto \N \Big( 0, \, \nabla_\lambda  g_T(\lambda^*, \pi_1, \pi_2)^\top \Sigma_{\lambda^*,\pi_0} \nabla_\lambda g_T(\lambda^*, \pi_1, \pi_2) \Big) \quad \TN{as} \quad \Toff \to \infty,
    \end{align*}
    uniformly over the underlying  $P_{\lambda^*, \pi_0} \in \{ P_{\lambda, \pi_0} : \lambda \in \Lambda \}$.  
    Thus, assuming a non-degenerate limiting variance, i.e., $\inf_{\lambda^* \in \Lambda} \nabla_\lambda g_T(\lambda^*, \pi_1, \pi_2)^\top \Sigma_{\lambda^*, \pi_0} \nabla_\lambda g_T(\lambda^*, \pi_1, \pi_2) > 0$, then for any $\alpha \in (0, 1)$,%
    \begin{align*}
        \lim_{\Toff \to \infty} \inf_{\lambda^* \in \Lambda} \PP_{P_{\lambda^*}, \pi_0} \bigg( \delta_T^* \in \bigg[ \hat{\delta}_T \pm \frac{ z_{1-\alpha/2} }{ \sqrt{ \Toff} } \sqrt{ \nabla_\lambda g_T(\hat{\lambda}, \pi_1, \pi_2)^\top \hat\Sigma_{\hat{\lambda},\pi_0} \nabla_\lambda g_T(\hat{\lambda}, \pi_1, \pi_2)  } 
        \bigg] \bigg) \geq 1-\alpha.
    \end{align*}
where $\hat\Sigma_{\hat{\lambda},\pi_0}$ is  any  positive definite matrix satisfying $\hat\Sigma_{\hat{\lambda},\pi_0} \Pto \Sigma_{\lambda^*, \pi_0}$, uniformly over $\lambda^\ast \in \Lambda$. 
\end{customcor}

\begin{proof}
The result holds by an argument very similar to the proof of Theorem \ref{thm:normality}.

By Assumption \ref{assump:lambdaNormality}, $\sqrt{ \Toff } ( \hat{\lambda} - \lambda^* ) \Dto N \big(0, \Sigma_{\lambda^*,\pi_0} \big)$ uniformly over $P_{\lambda^*} \in \{ P_{\lambda} : \lambda \in \Lambda \}$. By Slutsky's Theorem,
\begin{align}
    \nabla g_T(\lambda^*, \pi_1, \pi_2) \sqrt{ \Toff } \big( \hat{\lambda} - \lambda^* \big) \Dto \N \big(0, \nabla g_T(\lambda^*, \pi_1, \pi_2)^\top \Sigma_{\lambda^*,\pi_0} \nabla g_T(\lambda^*, \pi_1, \pi_2) \big)
    \label{eqn:intermed0B}
\end{align}
uniformly over $P_{\lambda^*, \pi_0} \in \{ P_{\lambda, \pi_0} : \lambda \in \Lambda \}$. By Lemma \ref{lemma:linearization},
\begin{align}
    \sqrt{ \Toff } \big( g_T( \hat{\lambda}, \pi_1, \pi_2 ) - g_T( \lambda^*, \pi_1, \pi_2) \big) - \nabla g_T(\lambda^*, \pi_1, \pi_2) \sqrt{ \Toff } \big( \hat{\lambda} - \lambda^* \big) \Pto 0,
    \label{eqn:intermed1B}
\end{align}
uniformly over $P_{\lambda^*, \pi_0} \in \{ P_{\lambda, \pi_0} : \lambda \in \Lambda \}$. The asymptotic normality result in the Theorem holds by \eqref{eqn:intermed0B}, \eqref{eqn:intermed1B}, and Slutsky's theorem. 

Thus, we have that
\begin{align*}
    \left( \nabla_\lambda g_T(\lambda^*, \pi_1, \pi_2)^\top \Sigma_{\lambda^*,\pi_0}  \nabla_\lambda g_T(\lambda^*, \pi_1, \pi_2) \right)^{-1/2}
    \sqrt{ \Toff } (\hat{\delta}_T - \delta_T^*) \Dto \N \left( 0, 1 \right),
\end{align*}
uniformly over $P_{\lambda^*, \pi_0} \in \{ P_{\lambda, \pi_0} : \lambda \in \Lambda \}$.
For the confidence interval result, by Slutsky's Theorem it is sufficient to show that
\begin{align*}
    \frac{
    \nabla_\lambda g_T(\hat\lambda, \pi_1, \pi_2)^\top \hat\Sigma_{\hat\lambda,\pi_0}  \nabla_\lambda g_T(\hat\lambda, \pi_1, \pi_2) }{
    \nabla_\lambda g_T(\lambda^*, \pi_1, \pi_2)^\top \Sigma_{\lambda^*,\pi_0} \nabla_\lambda g_T(\lambda^*, \pi_1, \pi_2) 
    }
    \Pto 1,
\end{align*}
uniformly over $P_{\lambda^*, \pi_0} \in \{ P_{\lambda, \pi_0} : \lambda \in \Lambda \}$. Since the denominator $\nabla_\lambda g_T(\lambda^*, \pi_1, \pi_2)^\top \Sigma_{\lambda^*,\pi_0}  \nabla_\lambda g_T(\lambda^*, \pi_1, \pi_2)$ is uniformly bounded above a constant since by our non-degeneracy assumption, \\
$\inf_{\lambda^* \in \Lambda} \nabla_\lambda g_T(\lambda^*, \pi_1, \pi_2)^\top \Sigma_{\lambda^*, \pi_0} \nabla_\lambda g_T(\lambda^*, \pi_1, \pi_2) > 0$, it sufficient to show that 
\begin{align}
    \label{eqn:consistentVarB}
    \nabla_\lambda g_T(\hat\lambda, \pi_1, \pi_2)^\top \hat\Sigma_{\hat\lambda,\pi_0}  \nabla_\lambda g_T(\hat\lambda, \pi_1, \pi_2) \Pto \nabla_\lambda g_T(\lambda^*, \pi_1, \pi_2)^\top \Sigma_{\lambda^*,\pi_0} \nabla_\lambda g_T(\lambda^*, \pi_1, \pi_2),
\end{align}
uniformly over $P_{\lambda^*, \pi_0} \in \{ P_{\lambda, \pi_0} : \lambda \in \Lambda \}$. \eqref{eqn:consistentVarB} above holds by Slutsky's Theorem since $\hat\Sigma_{\hat\lambda,\pi_0}$ is uniformly consistent, and since $\nabla_\lambda g_T(\hat\lambda, \pi_1, \pi_2) \Pto \nabla_\lambda g_T(\lambda^*, \pi_1, \pi_2)$ uniformly (which holds by continuous mapping theorem since $\hat\lambda \Pto \lambda^*$ uniformly and $\nabla g_T(\lambda, \pi_1, \pi_2)$ is continuous by Assumption \ref{assump:differentiable}).
\end{proof}

\subsection{Proof of Proposition \ref{prop:regretClose}: Value Difference Bounded by Regret}
\label{app:closeRegret}
\begin{customprop}{\ref{prop:regretClose}}[Value Difference Bounded by Regret]
    For any algorithm $\pi_1$,
    \begin{align*}
        \big| \theta^{\Gn}_T - \theta_T^* \big| \leq \frac{2}{T} \sup_{P' \in \MC{P}_{\lambda^*}} \MC{R}_T(P', \pi_1).
    \end{align*}
    Suppose we can construct an asymptotically valid confidence interval for $\theta_T^{\Gn}$ for any $\alpha \in (0, 1)$, i.e., \\
    $\lim_{\Toff \to \infty} \inf_{\lambda^* \in \Lambda} \PP_{P_{\lambda^*}, \pi_0} \left( \theta_T^{\Gn} \in \left[ \hat\theta_T \pm \nu_{\Toff}(\alpha) \right] \right) \geq 1-\alpha$, then we can construct the following asymptotically valid confidence interval for $\theta^*_T$:
    \begin{align*}
        \lim_{\Toff \to \infty} \, \inf_{\lambda^* \in \Lambda} \PP_{P_{\lambda^*}, \pi_0} \bigg( \theta_T^* \in \bigg[ \hat\theta_T \pm \bigg( \nu_{\Toff}(\alpha) + \frac{2}{T} \sup_{P' \in \MC{P}_{\lambda^*}} \MC{R}_T(P', \pi_1) \bigg) \bigg] \bigg) \geq 1-\alpha.
    \end{align*}
\end{customprop}

\begin{proof}
\bo{Showing the first statement of Proposition \ref{prop:regretClose}.}
By definition
\begin{align*}
    \big| \theta^{\Gn}_T - \theta_T^* \big|
    = \left| \E_{P_{\lambda^*}^\Gn, \pi_1} \left[ \frac{1}{T} \sum_{t=1}^T R_t \right] - \E_{P, \pi_1} \left[ \frac{1}{T} \sum_{t=1}^T R_t \right] \right|
\end{align*}
Note that both $P$ and $P_{\lambda^*}^\Gn$ have same mean rewards, we have $\E_{P_{\lambda^*}^\Gn} \big[ \frac{1}{T} \sum_{t=1}^T R_t(a^*) \big] = \E_{P} \big[ \frac{1}{T} \sum_{t=1}^T R_t(a^*) \big]$, where $a^\ast = \TN{argmax}_{a \in \MC{A}} \mu_a$ is an optimal action. Thus by adding and subtracting $\E_{P} \big[ \frac{1}{T} \sum_{t=1}^T R_t(a^*) \big]$ and applying the triangle inequality we obtain
\begin{align*}
    \big| \theta^{\Gn}_T - \theta_T^* \big| &\leq \bigg| \E_{P_{\lambda^*}^\Gn, \pi_1} \bigg[ \frac{1}{T} \sum_{t=1}^T R_t \bigg] - \E_{P_{\lambda^*}^\Gn} \bigg[ \frac{1}{T} \sum_{t=1}^T R_t(a^*) \bigg] \bigg| 
    + \bigg| \E_{P} \bigg[ \frac{1}{T} \sum_{t=1}^T R_t(a^*) \bigg] - \E_{P, \pi_1} \bigg[ \frac{1}{T} \sum_{t=1}^T R_t \bigg] \bigg| \\
    &\leq \frac{2}{T} \sup_{P' \in \MC{P}_{\lambda^*}} \mathcal{R}_T(P', \pi_1).
\end{align*}
The final inequality holds by the definition of regret.

\bo{Showing the second statement of Proposition \ref{prop:regretClose}.}
By the assumption of the proposition,
\begin{align*}
    &1-\alpha \leq \lim_{\Toff \to \infty} \, \inf_{\lambda^* \in \Lambda} \PP_{P_{\lambda^*}, \pi_0} \Big( \theta_T^{\Gn} \in \Big[ \hat\theta_T \pm \nu_{\Toff}(\alpha) \Big] \Big) \\
    &= \lim_{\Toff \to \infty} \, \inf_{\lambda^* \in \Lambda} \PP_{P_{\lambda^*}, \pi_0} \Big( \big| \theta_T^{\Gn} - \hat\theta_T \big| \leq \nu_{\Toff}(\alpha) \Big)
\end{align*}
By the first statement of the proposition, i.e., that $\big| \theta^{\Gn}_T - \theta_T^* \big| \leq \frac{2}{T} \sup_{P' \in \MC{P}_{\lambda^*}} \MC{R}_T(P', \pi_1)$, we have that
\begin{align*}
    &= \lim_{\Toff \to \infty} \, \inf_{\lambda^* \in \Lambda} \PP_{P_{\lambda^*}, \pi_0} \bigg( \big| \theta_T^{\Gn} - \hat\theta_T \big| + \big| \theta_T^* - \theta_T^{\Gn} \big| \leq \nu_{\Toff}(\alpha) + \frac{2}{T} \sup_{P' \in \MC{P}_{\lambda^*}} \mathcal{R}_T(P', \pi_1) \bigg)
\end{align*}
Since $\big| \theta_T^* - \hat\theta_T \big| \leq \big| \theta_T^* - \theta_T^{\Gn} \big| + \big| \theta_T^{\Gn} - \hat\theta_T \big|$,
\begin{align*}
    &\leq \lim_{\Toff \to \infty} \, \inf_{\lambda^* \in \Lambda} \PP_{P_{\lambda^*}, \pi_0} \bigg( \big| \theta_T^* - \hat\theta_T \big| \leq \nu_{\Toff}(\alpha) + \frac{2}{T} \sup_{P' \in \MC{P}_{\lambda^*}} \mathcal{R}_T(P', \pi_1) \bigg) \\
    &= \lim_{\Toff \to \infty} \, \inf_{\lambda^* \in \Lambda} \PP_{P_{\lambda^*}, \pi_0} \bigg( \theta_T^* \in \bigg[ \hat\theta_T \pm \bigg( \nu_{\Toff}(\alpha) + \frac{2}{T} \sup_{P' \in \MC{P}_{\lambda^*}} \MC{R}_T(P', \pi_1) \bigg) \bigg] \bigg).
\end{align*}
\end{proof}

\subsection{Gradient under Common Reward Distributions}
\label{app:exampleDistributions}

\begin{customex}{\ref{ex:bernoulli}}[Bernoulli $P_\lambda$]
    Here $\lambda = \{\lambda_a\}_{a \in \MC{A}}$ where $\lambda_a = \mu_a$ represents the mean reward after taking action $a$.
    \begin{itemize}[leftmargin=*]
        \item $P_{\lambda_a}(R_t \mid A_t = a) = \mu_a^{R_t} (1 - \mu_a)^{1-R_t}$
        \item $\log P_{\lambda_a}(R_t \mid A_t = a) = R_t \log \mu_a + (1-R_t) \log (1 - \mu_a)$
        \item $\nabla_{\lambda_a} \log P_{\lambda_a}(R_t \mid A_t = a) 
        = \frac{ R_t }{\mu_a} - \frac{1-R_t}{1 - \mu_a} 
        = \frac{R_t (1 - \mu_a) - (1-R_t) \mu_a}{\mu_a (1 - \mu_a)}
        = \frac{R_t - \mu_a}{\mu_a (1 - \mu_a)}$
    \begin{align*}
        \nabla_{\lambda} f_T(\lambda, \pi_1) 
        &= \frac{1}{T} \sum_{t=1}^{T} \E_{\lambda, \pi_1} \begin{bmatrix}
            \mathbbm{1}(A_t = 1) \frac{R_t - \mu_{1}}{\mu_{1} (1 - \mu_{1})} G_t \\
            \vdots \\
            \mathbbm{1}(A_t = |\MC{A}|) \frac{R_t - \mu_{|\MC{A}|}}{\mu_{|\MC{A}|} (1 - \mu_{|\MC{A}|})} G_t
        \end{bmatrix} \nonumber \\
        &= \frac{1}{T} \sum_{t=1}^{T} \E_{\lambda, \pi_1} 
        \begin{bmatrix}
            \mathbbm{1}(A_t = 1) \left( 1 + \frac{R_t - \mu_{1}}{\mu_{1} (1 - \mu_{1})} G_{t+1} \right) \\
            \vdots \\
            \mathbbm{1}(A_t = |\MC{A}|) \left( 1 + \frac{R_t - \mu_{|\MC{A}|}}{\mu_{|\MC{A}|} (1 - \mu_{|\MC{A}|})} G_{t+1} \right)
        \end{bmatrix}
    \end{align*}
    The last equality above holds since $\E \left[ \frac{ (R_t - \mu_a) }{\mu_a (1-\mu_a)} R_t \mid A_t = a \right] 
    = \E \left[ \frac{ (R_t - \mu_a) }{\mu_a (1-\mu_a)} (R_t - \mu_a + \mu_a) \mid A_t = a \right] 
    = \E \left[ \frac{(R_t - \mu_a)^2}{\mu_a (1-\mu_a)} \mid A_t = a \right] = 1$.
    \end{itemize}
\end{customex}

\begin{customex}{\ref{ex:gaussian}}[Gaussian $P_\lambda$]
    Here $\lambda = \{\lambda_a\}_{a \in \MC{A}}$ where $\lambda_a = (\mu_a, \sigma_a^2)$ represents the mean and variance of the reward after taking action $a$.
    \begin{itemize}[leftmargin=*]
        \item $P_{\lambda_a}(R_t \mid A_t = a) = \frac{1}{ \sqrt{2 \pi \sigma_a^2}} \exp \left( -\frac{(R_t - \mu_a)^2}{ 2 \sigma_a^2 } \right)$
        \item $\log P_{\lambda_a}(R_t \mid A_t = a) = - \frac{1}{2} \log(2 \pi) - \frac{1}{2} \log(\sigma_a^2) - \frac{(R_t - \mu_a)^2}{ 2 \sigma_a^2 }$
        \item $\nabla_{\lambda_a} \log P_{\lambda_a}(R_t \mid A_t = a)
        = \begin{pmatrix}
            \frac{\partial}{\partial \mu_a} \log P_{\lambda_a}(R_t \mid A_t = a) \\
            \frac{\partial}{\partial \sigma_a^2} \log P_{\lambda_a}(R_t \mid A_t = a)
        \end{pmatrix}
        = \begin{pmatrix}
            \frac{R_t - \mu_a}{\sigma_a^2} \\
            -\frac{1}{2 \sigma_a^2} + \frac{(R_t - \mu_a)^2}{ 2 \sigma_a^4 }
        \end{pmatrix}$
        \item Thus, for $G_t \triangleq \sum_{t'=t}^T R_{t'}$,
    \begin{align*}
        \nabla_{\lambda} f_T(\lambda, \pi_1) 
        &= \frac{1}{T} \sum_{t=1}^{T} \E_{\lambda, \pi_1} \begin{bmatrix}
            \mathbbm{1}(A_t = 1) \frac{R_t - \mu_1}{\sigma_1^2} G_t \\
            \mathbbm{1}(A_t = 1) \left( -\frac{1}{2\sigma_1^2} + \frac{(R_t - \mu_1)^2}{ 2 \sigma_1^4 } \right) G_t \\
            \vdots \\
            \mathbbm{1}(A_t = |\MC{A}|) \frac{R_t - \mu_{|\MC{A}|}}{\sigma_{|\MC{A}|}^2} G_t \\
            \mathbbm{1}(A_t = |\MC{A}|) \left( -\frac{1}{2\sigma_{|\MC{A}|}^2} + \frac{(R_t - \mu_{|\MC{A}|})^2}{ 2 \sigma_{|\MC{A}|}^4 } \right) G_t
        \end{bmatrix} \nonumber \\
        &= \frac{1}{T} \sum_{t=1}^{T} \E_{\lambda, \pi_1} \begin{bmatrix}
            \mathbbm{1}(A_t = 1) \left( 1 + \frac{R_t - \mu_1}{\sigma_1^2} G_{t+1} \right) \\
            \mathbbm{1}(A_t = 1) \left( -\frac{1}{2 \sigma_1^2} + \frac{(R_t - \mu_1)^2}{ 2 \sigma_1^4 } \right) G_{t+1} \\
            \vdots \\
            \mathbbm{1}(A_t = |\MC{A}|) \left( 1 + \frac{R_t - \mu_{|\MC{A}|}}{\sigma_{|\MC{A}|}^2} G_{t+1} \right) \\
            \mathbbm{1}(A_t = |\MC{A}|) \left( -\frac{1}{2 \sigma_{|\MC{A}|}^2} + \frac{(R_t - \mu_{|\MC{A}|})^2}{ 2 \sigma_{|\MC{A}|}^4 } \right) G_{t+1}
        \end{bmatrix}
    \end{align*}    
    The last equality above holds since (i) $\E \left[ \frac{ (R_t - \mu_a) }{\sigma_a^2} R_t \mid A_t = a \right] 
    = \E \left[ \frac{ (R_t - \mu_a) }{\sigma_a^2} ( R_t - \mu_a + \mu_a) \mid A_t = a \right] 
    = \E \left[ \frac{ (R_t - \mu_a)^2 }{\sigma_a^2} \mid A_t = a \right] = 1$, and 
    (ii) $\E \left[ \left( -\frac{1}{2\sigma_a^2} + \frac{(R_t - \mu_a)^2}{ 2 \sigma_a^4 } \right) R_t \mid A_t = a \right] 
    = -\frac{\mu_a}{2\sigma_a^2} + \E \left[ \frac{(R_t - \mu_a)^2}{ 2 \sigma_a^4 } R_t \mid A_t = a \right] 
    = -\frac{\mu_a}{2\sigma_a^2} + \E \left[ \frac{(R_t - \mu_a)^2 \mu_a}{ 2 \sigma_a^4 } + \frac{(R_t - \mu_a)^3}{ 2 \sigma_a^4 } \mid A_t = a \right]
    = -\frac{\mu_a}{2\sigma_a^2} + \frac{\mu_a}{2\sigma_a^2} + 0 = 0$.
    
    \end{itemize}
\end{customex}

\begin{example}[Natural Exponential Family $P_\lambda$]
    \label{ex:exponentialFamily}
    Here $\lambda = \{\lambda_a\}_{a \in \MC{A}}$.
    \begin{itemize}[leftmargin=*]
        \item $P_{\lambda_a}(R_t \mid A_t = a) = h \big( R_t(a) \big) \exp \left( \lambda_a^\top \phi \big( R_t(a) \big) - \Phi (\lambda_a) \right)$
        \item $\log P_{\lambda_a}(R_t \mid A_t = a) = \log h \big( R_t(a) \big) + \lambda_a^\top \phi \big( R_t(a) \big) - \Phi (\lambda_a)$
        \item $\nabla_{\lambda_a} \log P_{\lambda_a}(R_t \mid A_t = a)
        = \phi \big( R_t(a) \big) - \nabla_{\lambda_a} \Phi (\lambda_a)$
        \item Thus, for $G_t \triangleq \sum_{t'=t}^T R_{t'}$,
    \begin{align*}
        \nabla_{\lambda} f_T(\lambda, \pi_1) 
        &= \frac{1}{T} \sum_{t=1}^{T} \E_{\lambda, \pi_1} \begin{bmatrix}
            \mathbbm{1}(A_t = 1) \left( \phi \big( R_t(1) \big) - \frac{\partial}{\partial \lambda_1} \Phi (\lambda_1) \right) G_t \\
            \vdots \\
            \mathbbm{1}(A_t = |\MC{A}|) \left( \phi \big( R_t(|\MC{A}|) \big) - \frac{\partial}{\partial \lambda_{|\MC{A}|}} \Phi (\lambda_{|\MC{A}|}) \right) G_t 
        \end{bmatrix} \nonumber 
    \end{align*}
    \end{itemize}
\end{example}

  \newpage
  \section{Additional Simulations Details and Results}

\subsection{Additional Simulation Results}

 We consider two reward settings: Gaussian and sub-Gaussian. In both cases, the behavior policy is taken to be non-adaptive, while the evaluation policy is adaptive.

\begin{figure}[h] 
    \centering
    \includegraphics[width=0.8\textwidth,trim={1cm 0.5cm 1cm 0.5cm}]{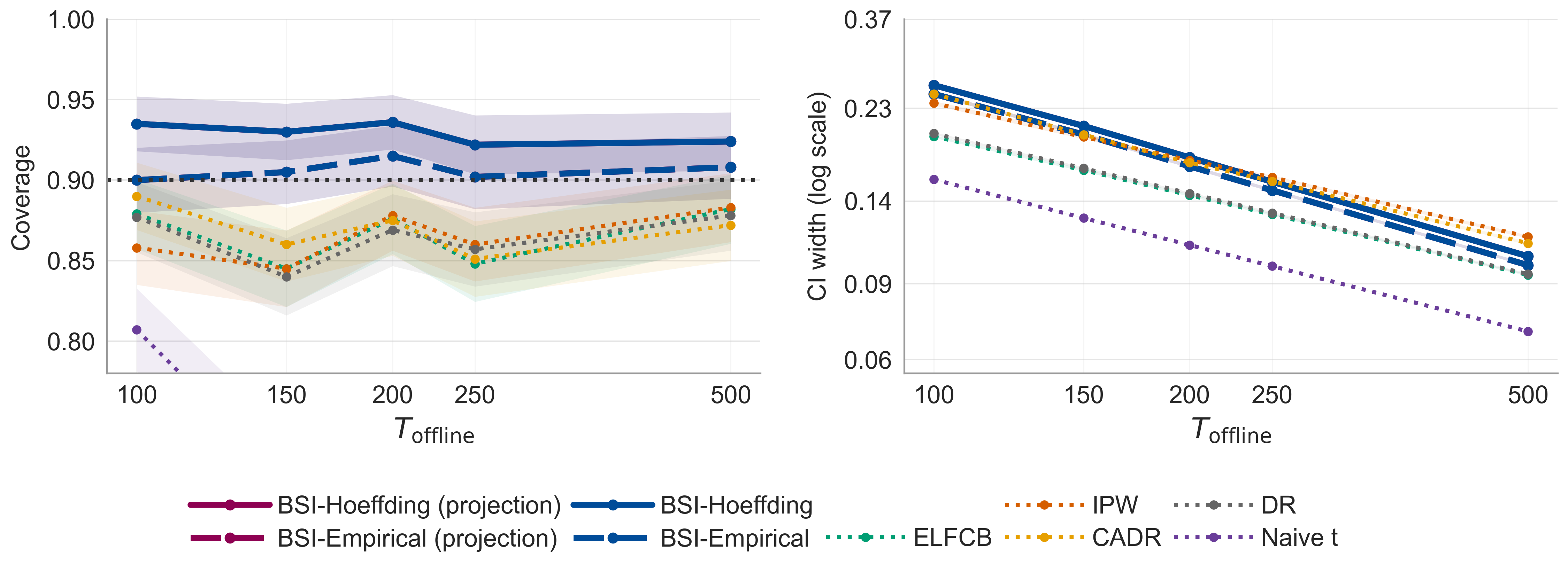}
    \caption{\bo{Sub-Gaussian Reward Distribution Experiments.} Average coverage and width of $90\%$ confidence intervals. All experiments use $T=500$ and are averaged across 1000 replications (shaded regions denote $1$ standard error). Three-arm setting where the true environment rewards are Beta-distributed with hyper-parameters $(0.35,0.65), (0.5,0.5)$ and $ (0.5,0.5)$. The behavior policy is uniform, and the evaluation policy is Thompson Sampling (Gaussian-Gaussian). BSI internally uses $m=20k$ MC repetitions. True parameter $\theta^\star$ is obtained via $50{,}000$ independent MC replications of the true environment.}
    \label{fig:beta-app}
\end{figure}

\begin{figure}[h]
    \centering
    \includegraphics[width=0.8\textwidth,trim={1cm 0.5cm 1cm 0.5cm}]{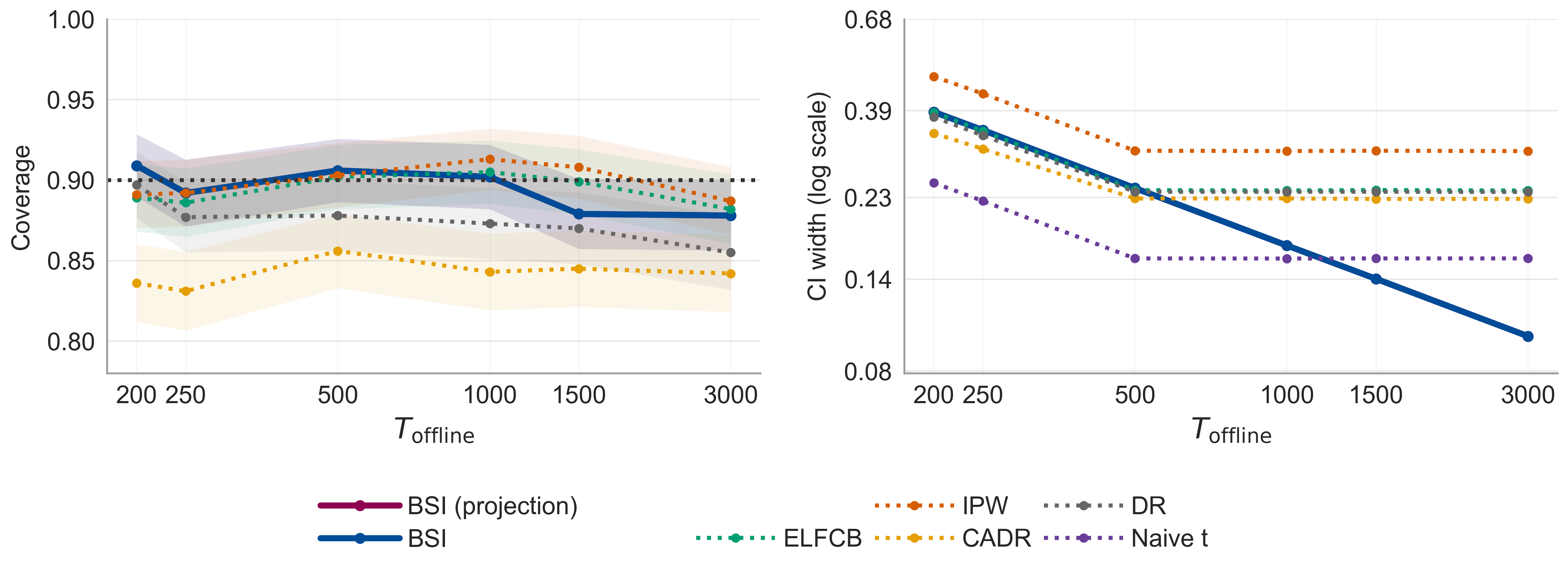}
    \caption{\bo{Gaussian Reward Distribution Experiments.} Average coverage and width of $90\%$ confidence intervals. All experiments use $T=500$ and are averaged across 1000 replications (shaded regions denote $1$ standard error). Three-arm setting with Gaussian rewards with means $(0.1,0.2,1)$ and unit variance. The variance are assumed to be unknown and estimated from the data. The behavior policy is uniform, and the evaluation policy is $\epsilon$-Greedy. True parameter $\theta^\star$ is obtained via $50{,}000$ independent MC replications of the true environment.}
    \label{fig:gaus-app}
\end{figure}

\subsection{Semi-Synthetic Environment Construction}
\label{app:semiSyntheticEnv}

To evaluate the empirical performance of BSI in a more realistic setting, we construct a synthetic simulator based on real world data. The data comes from a randomized controlled trial of interactive voice-response mobile phone surveys conducted in Bangladesh and Uganda \citep{gibson2022promised}. Adult respondents reached through automated dialing were randomized to one of three airtime incentive arms: no incentive, a promised airtime incentive, or a lottery-based airtime incentive. The study measures survey participation outcomes, primarily cooperation and response rates, in order to assess how different incentive schemes affect engagement with mobile health surveys. 

For our semi-synthetic experiments, we focus on the cooperation rates of Uganda. The paper~\citet{gibson2022promised} reports that the cooperation rate for the incentives to be $63.4\%$, $76.6\%$, and $72.2\%$ (corresponding to no response, airtime incentive and lottery based incentive respectively). We interpret these as chances of a randomly selected individual to prefer one of the listed incentives.

Using these reported rates, we generate synthetic data from a three-armed Bernoulli bandit under a uniform behavior policy, and use this setup to evaluate the performance of our method for inference on the average value of Thompson sampling when treated as the evaluation policy. Figure~\ref{fig:semi_synthetic} highlights that BSI achieves nominal coverage for $\alpha = 0.05, 0.1$, whereas the baseline methods have significant undercoverage.

\subsection{Additional Information on Adaptive Algorithms}
\label{app:adaptiveAlgs}

For our experiments, we consider three different adaptive policies, $\epsilon$-greedy,  Gaussian Thompson sampling and Bernoulli Thompson sampling. We explain each of these below.

\bo{$\epsilon$-greedy.}
In Figures~\ref{fig:motivatingIssue} (left), ~\ref{fig:Simulation} (left) and~\ref{fig:gaus-app}  we use an $\epsilon$-greedy. For the first $|\MC{A}|$ rounds, each arm is pulled once. For $t>|\MC{A}|$,
\[
\pi_t(a \mid \MC{H}_{t-1})
=
\begin{cases}
1-\epsilon+\epsilon/|\MC{A}|, & a \in \TN{argmax}_{a'} \widehat{\mu}_{a'}(t-1),\\[3pt]
\epsilon/|\MC{A}|, & \text{otherwise},
\end{cases}
\]
where
\[
\widehat{\mu}_a(t-1)
=
\frac{\sum_{s=1}^{t-1} R_s \mathbbm{1}\{A_s=a\}}
     {\sum_{s=1}^{t-1} \mathbbm{1}\{A_s=a\}}.
\]
Throughout, we use $\epsilon=0.1$. 

\bo{Thompson sampling (Beta-Bernoulli).}
For the simulations in Figures~\ref{fig:Simulation} (center and right) and \ref{fig:semi_synthetic}, we use Thompson Sampling with a Beta-Bernoulli model. The Beta distribution has prior parameters $(\alpha_0,\beta_0)=(1,1)$ for each arm. The Beta posterior for each arm is 
\[
\theta_a \mid \MC{H}_{t-1} \sim \mathrm{Beta}(\alpha_{a,t-1},\beta_{a,t-1}).
\]
Thompson Sampling selects action \(a\) according to its posterior probability of being optimal:
\begin{align*}
    \widetilde{\pi}_t(a \mid \mathcal H_{t-1})
=
\mathbb P\!\left(
a = \arg\max_{a' \in \mathcal A} \theta_{a'}
\mid
\mathcal H_{t-1}
\right).
\end{align*}
To guarantee sufficient exploration, we enforce a minimum action probability \(\pi_{\min}=0.01\) by clipping the sampling probabilities:
\begin{align}
    \pi_t(a \mid \mathcal H_{t-1})
    &=
    \pi_{\min}
    +
    (1-|\mathcal A|\pi_{\min}) \frac{
    \max\!\left\{
    \widetilde{\pi}_t(a \mid \mathcal H_{t-1})-
    \pi_{\min},
    0
    \right\}
    }{
    \sum_{a\in\mathcal A}
    \max\!\left\{
    \widetilde{\pi}_t(a \mid \mathcal H_{t-1})-
    \pi_{\min},
    0
    \right\}
    }.
    \label{eqn:clipp}
\end{align}
By construction it follows that $\sum_{a\in\mathcal A}
\pi_t(a \mid \mathcal H_{t-1}) = 1$.

\bo{Thompson sampling (Gaussian-Gaussian).}
For the simulations in Figures~\ref{fig:motivatingIssue} (right),~\ref{fig:Simulation} (center, right),~\ref{fig:BetaEqarmMain} and~\ref{fig:beta-app}, we use Thompson Sampling with a Gaussian-Gaussian model. The Bayesian model has parameters $(m_0,v_0,\sigma)=(0,1,1)$ for each arm, where $m_0$ is the prior mean, $v_0$ is the prior variance, and $\sigma$ is the reward noise standard deviation.
The Gaussian posterior for each arm is 
\[
\theta_a \mid \MC{H}_{t-1} \sim N(m_{a,t-1},v_{a,t-1}).
\]
Thompson Sampling selects an action a according to the posterior probability it is optimal:
\begin{align*}
    \tilde{\pi}_t(a \mid \HH_{t-1}) = \PP( a = \TN{argmax}_{a \in \MC{A}} \theta_a \mid \HH_{t-1}).
\end{align*}
We further clip the action selection probabilities to ensure that each action is selected with a minimum probability of $\pi_{\min} = 0.01$ using display \eqref{eqn:clipp}.

\subsection{Additional BSI Implementation Details}
\label{app:BSIdetails}

\subsubsection{BSI for Sub-Gaussian Rewards Details}

The \textsc{BSI-Hoeffding} approach uses the fact that the rewards are bounded in $[0,1]$ and sets
\[
(\sigma_a^*)^2 = 1/4
\]
for every arm. 

The \textsc{BSI-Empirical} approach treats the variance as unknown and estimates an arm-specific variance using $D^{(\pi_0)}$. The estimator $\hat{\sigma}_a^2$ is formed differently depending on whether $\pi_0$ is adaptive or not; see Sections \ref{app:nonadaptive_sig_comp} and \ref{app:adaptive_sig_comp}.

\subsubsection{Asymptotic Normality of $\hat{\lambda}$ under Static $\pi_0$}
\label{app:nonadaptive_sig_comp}

\paragraph{Gaussian (or Sub-Gaussian) Rewards.} 
In this setting,
$\lambda=(\mu_1,\sigma_1^2,\ldots,\mu_{|\mathcal{A}|},\sigma_{|\mathcal{A}|}^2)$. We use $N_a=\sum_{t=1}^{T_{\mathrm{offline}}}\mathbbm{1}(A_t=a)$,  $\hat{\mu}_a = \frac{\sum_{t=1}^{\Toff} R_t}{N_a}$, and $\hat{\sigma}_a^2 = \frac{1}{N_a} \sum_{t=1}^{\Toff} \mathbbm{1}(A_t=a) (R_t - \hat{\mu}_a)^2$.
Under a static $\pi_0$,
\begin{align*}
    \sqrt{T_{\mathrm{offline}}}
    \begin{pmatrix}
        \hat{\mu}_1-\mu_1^* \\
        \hat{\sigma}_1^2-(\sigma_1^*)^2 \\
        \vdots \\
        \hat{\mu}_{|\MC{A}|}-\mu_{|\MC{A}|}^* \\
        \hat{\sigma}_{|\MC{A}|}^2-(\sigma_{|\MC{A}|}^*)^2 
    \end{pmatrix}
    \xrightarrow{D}
    \mathcal{N}\!\left(0, \TN{BlockDiagonal}(\Sigma_{\pi_0}^{(a)})\right),
\end{align*}
where
\begin{align*}
    \Sigma_{\pi_0}^{(a)}
    =
    \pi_0(a)^{-1}
    \begin{pmatrix}
        (\sigma_a^*)^2 & 0 \\
        0 & \mathbb{E}\!\left[(R_t(a)-\mu_a^*)^4\right]-(\sigma_a^*)^4
    \end{pmatrix}.
\end{align*}

\paragraph{Bernoulli Rewards.}
In this setting, $\lambda=(\mu_1,\ldots,\mu_{|\mathcal{A}|})$. 
We use $N_a=\sum_{t=1}^{T_{\mathrm{offline}}}\mathbbm{1}(A_t=a)$ and $\hat{\mu}_a = \frac{\sum_{t=1}^{\Toff} R_t}{N_a}$. Under a static $\pi_0$,
\begin{align*}
    \sqrt{T_{\mathrm{offline}}}\,\begin{pmatrix}
        \hat{\mu}_1-\mu_1^* \\
        \vdots \\
        \hat{\mu}_{|\MC{A}|}-\mu_{|\MC{A}|}^*
    \end{pmatrix}
    \xrightarrow{D}
    \mathcal{N}\!\left(0,\,\TN{Diagonal}(\Sigma_{\pi_0}^{(a)})\right),
\end{align*}
where
\begin{align*}
    \Sigma_{\pi_0}^{(a)}
    =
    \pi_0(a)^{-1}\,\mu_a^*(1-\mu_a^*).
\end{align*}

\subsubsection{Asymptotic Normality of $\hat{\lambda}$ under Adaptive $\pi_0$}
\label{app:adaptive_sig_comp}

\paragraph{Gaussian (or Sub-Gaussian) Rewards.}
We use the following M-estimation function:
\begin{align*}
    m_{\lambda}(A_t,R_t)
    =
    -\sum_{a\in\mathcal{A}}\mathbbm{1}(A_t=a)
    \left\{\frac{(R_t-\mu_a)^2}{2\sigma_a^2}+\frac{1}{2}\log\sigma_a^2\right\},
\end{align*}
with score $\dot{m}_\lambda(A_t,R_t)\triangleq\nabla_\lambda m_\lambda(A_t,R_t)=\big(\dot{m}_{\lambda_1}(A_t,R_t),\ldots,\dot{m}_{\lambda_{|\mathcal{A}|}}(A_t,R_t)\big)^\top$, where
\[\dot{m}_{\lambda_a}(A_t,R_t)
    =\mathbbm{1}(A_t=a)\begin{pmatrix}
            \dfrac{R_t-\mu_a}{\sigma_a^2}\\[0.5em]
            \dfrac{(R_t-\mu_a)^2}{2\sigma_a^4}-\dfrac{1}{2\sigma_a^2}
        \end{pmatrix}.
\]
The estimator $\hat{\lambda}_a=(\hat{\mu}_a,\hat{\sigma}_a^2)$ solves
$0=\sum_{t=1}^{T_{\mathrm{offline}}}W_t\dot{m}_{\lambda_a}(A_t,R_t)$,
where $W_t=1/\sqrt{\pi_t(A_t\mid \MC{H}_{t-1})}$, giving
\begin{align*}
    \hat{\mu}_a
    =
    \frac{\sum_{t=1}^{T_{\mathrm{offline}}}W_t\mathbbm{1}(A_t=a)R_t}
         {\sum_{t=1}^{T_{\mathrm{offline}}}W_t\mathbbm{1}(A_t=a)}
    \qquad\text{and}\qquad
    \hat{\sigma}_a^2
    =
    \frac{\sum_{t=1}^{T_{\mathrm{offline}}}W_t\mathbbm{1}(A_t=a)(R_t-\hat{\mu}_a)^2}
         {\sum_{t=1}^{T_{\mathrm{offline}}}W_t\mathbbm{1}(A_t=a)}.
\end{align*}
Using the asymptotic normality result from \citet{zhang2021statistical},
\begin{align}
\label{eq:asymp_adaptive_gau}
    \sqrt{T_{\mathrm{offline}}}\,\Sigma^{-1/2}\ddot{M}_{T_{\mathrm{offline}}}
    \left(\hat{\lambda}-\lambda^*\right)
    \xrightarrow{D}
    \mathcal{N}(0,I),
\end{align}
where
\begin{align*}
    \Sigma
    =
    \mathrm{BlockDiagonal}\!\left(
        \mathbb{E}\!\left[
            \frac{1}{T_{\mathrm{offline}}}\sum_{t=1}^{T_{\mathrm{offline}}}
            W_t^2\,\dot{m}_{\lambda_a}(A_t,R_t) \dot{m}_{\lambda_a}(A_t,R_t)^\top
        \right]
    \right)
    \in\mathbb{R}^{|\lambda|\times|\lambda|},
\end{align*}
Also,
$\ddot{M}_{T_{\mathrm{offline}}}=\frac{1}{T_{\mathrm{offline}}}\sum_{t=1}^{T_{\mathrm{offline}}}\mathrm{BlockDiagonal}\big(W_t\ddot{m}_{\lambda_a}(A_t,R_t)\big)$,
where
\begin{align*}
    \ddot{m}_{\lambda_a}(A_t,R_t)
    =
    \mathbbm{1}(A_t=a)
    \begin{pmatrix}
        \dfrac{-1}{\sigma_a^2} & -\dfrac{R_t-\mu_a}{\sigma_a^4}\\[0.5em]
        -\dfrac{R_t-\mu_a}{\sigma_a^4} & -\dfrac{(R_t-\mu_a)^2}{\sigma_a^6}+\dfrac{1}{2\sigma_a^4}
    \end{pmatrix}.
\end{align*}
Note that to satisfy Assumption~\ref{assump:lambdaNormalityAdaptive}, we use $V_{\text{offline}}=\sqrt{T_{\mathrm{offline}}}\,\Sigma^{-1/2}\ddot{M}_{T_{\mathrm{offline}}}$.

\paragraph{Bernoulli Rewards.}
We use the following M-estimation function:
\begin{align*}
    m_{\lambda}(A_t,R_t)
    =
    \sum_{a\in\mathcal{A}}\mathbbm{1}(A_t=a)
    \left\{R_t\log\mu_a+(1-R_t)\log(1-\mu_a)\right\},
\end{align*}
with score
\begin{align*}
    \dot{m}_{\lambda_a}(A_t,R_t)
    =
    \mathbbm{1}(A_t=a)\left(\frac{R_t-\mu_a}{\mu_a(1-\mu_a)}\right).
\end{align*}
The estimator $\hat{\lambda}_a=\hat{\mu}_a$ solves
$0=\sum_{t=1}^{T_{\mathrm{offline}}}W_t\dot{m}_{\lambda_a}(A_t,R_t)$,
where $W_t=1/\sqrt{\pi_t(A_t\mid \MC{H}_{t-1})}$, giving
\begin{align*}
    \hat{\mu}_a
    =
    \frac{\sum_{t=1}^{T_{\mathrm{offline}}}W_t\mathbbm{1}(A_t=a)R_t}
         {\sum_{t=1}^{T_{\mathrm{offline}}}W_t\mathbbm{1}(A_t=a)}.
\end{align*}
Using the asymptotic normality result from \citet{zhang2021statistical},
\begin{align*}
    \sqrt{T_{\mathrm{offline}}}\,\Sigma^{-1/2}\ddot{M}_{T_{\mathrm{offline}}}
    \left(\hat{\lambda}-\lambda^*\right)
    \xrightarrow{D}
    \mathcal{N}(0,I),
\end{align*}
where
\begin{align*}
    \Sigma
    =
    \mathrm{Diagonal}\!\left(
        \mathbb{E}\!\left[
            \frac{1}{T_{\mathrm{offline}}}\sum_{t=1}^{T_{\mathrm{offline}}}
            W_t^2\,\dot{m}_{\lambda_a}(A_t,R_t) \dot{m}_{\lambda_a}(A_t,R_t)^\top
        \right]
    \right),
\end{align*}
Also, $\ddot{M}_{T_{\mathrm{offline}}}=\frac{1}{T_{\mathrm{offline}}}\sum_{t=1}^{T_{\mathrm{offline}}}\mathrm{Diagonal}\big(W_t\ddot{m}_{\lambda_a}(A_t,R_t)\big)$,
where
\begin{align*}
    \ddot{m}_{\lambda_a}(A_t,R_t)
    =
    \mathbbm{1}(A_t=a)
    \left(-\frac{R_t}{\mu_a^2}-\frac{1-R_t}{(1-\mu_a)^2}\right).
\end{align*}

\subsection{Baseline Inference Methods} \label{app:baselines}

In this Section, we discuss some popular off-policy evaluation methods that construct confidence intervals around the average expected reward. We consider four baseline methods---\textit{Inverse Propensity Weighting (IPW)}~\citep{horvitz1952generalization,hirano2003efficient}, \textit{Doubly Robust (DR)} estimator~\citep{robins1994estimation,chernozhukov2018double,jiang2016doubly}, \textit{ Contextual Adaptive Doubly Robust (CADR)} estimator~\citep{bibaut2021post} and  \textit{Empirical Likelihood (ELFCB)}~\citep{karampatziakis2020empirical}. The IPW and DR estimators have guarantees when both the behavior and evaluation policy are static. In contrast, CADR and ELFCB allow the behavior policy is adaptive. However, none of the standard OPE methods assume that the evaluation policy $\pi_1(A \mid \MC{H}_{t-1})$ to be adaptive, i.e., they depend on the trajectory history $\MC{H}_{t-1}$.

In this work, throughout we implement these baseline methods using the following importance weights (when importance weights are needed in the method):
\[
W_t^{}
=
\frac{\pi_1(A_t \mid \MC{H}_{t-1})}{\pi_0(A_t \mid \MC{H}_{t-1})}.
\]

\subsubsection{Inverse Probability Weighting}

The inverse propensity weighting (IPW) baseline re-weights observed rewards by
the importance weights. 
The estimator is defined as
\[
\widehat{\theta}_{\mathrm{IPW}}
=
\frac{1}{T_{\text{offline}}}\sum_{t=1}^{T_{\text{offline}}} W_t R_t.
\]
The implementation then forms a confidence intervals as follows:
\[
\left[\widehat{\theta}_{\mathrm{IPW}}
\pm
z_{1-\alpha/2} \frac{\widehat{\sigma}}{\sqrt{T_{\mathrm{offline}}}} \right] \, \quad \text{where} \quad 
\widehat{\sigma}^2 \triangleq \frac{1}{T_{\mathrm{offline}}}\sum^{T_{\mathrm{offline}}}_{t=1}  (W_tR_t -\widehat{\theta}_{\mathrm{IPW}} )^2.
\]

\subsubsection{Doubly Robust}

The doubly robust (DR) estimator integrates a reward model with importance 
weighting to yield an estimator that remains consistent if either component 
is correctly specified. Let $\widehat{q}(a)$ denote the estimated mean reward 
for arm $a$, obtained from the offline data. For each round $t$, define the 
DR score
\[
\phi_t \;=\; \sum_{a \in \mathcal{A}} \pi_t(a \mid \mathcal{H}_{t-1})\,\widehat{q}(a) 
\;+\; W_t\bigl(R_t - \widehat{q}(A_t)\bigr),
\]
where $W_t = \pi_t(A_t)/\pi_{0,t}(A_t)$ is the importance weight at time $t$. 
The DR point estimate is then
\[
\widehat{\Psi}_{\mathrm{DR}} \;=\; \frac{1}{T_{\mathrm{off}}} 
\sum_{t=1}^{T_{\mathrm{off}}} \phi_t.
\]

\paragraph{Reward Model Estimation.}
The reward model $\widehat{q}(a)$ is estimated separately for each arm using 
the offline sample. Let
\[
N_a \;=\; \sum_{t=1}^{T_{\mathrm{off}}} \mathbf{1}\{A_t = a\}, 
\qquad 
\widehat{\mu}_a \;=\; \frac{1}{N_a} \sum_{t=1}^{T_{\mathrm{off}}} 
R_t\,\mathbf{1}\{A_t = a\}
\]
denote the arm-specific sample size and sample mean reward, respectively.

\textit{Bernoulli rewards.} A Beta prior is placed on the mean reward,
\[
q(a) \;\sim\; \mathrm{Beta}(\alpha_0,\,\beta_0),
\]
with default hyperparameters $(\alpha_0, \beta_0) = (1, 1)$. The reward model 
is taken to be the Beta--Binomial posterior mean,
\[
\widehat{q}(a) \;=\; \frac{N_a\,\widehat{\mu}_a + \alpha_0}{N_a + \alpha_0 + \beta_0}.
\]

\textit{Gaussian rewards.} A normal prior is placed on the mean reward,
\[
q(a) \;\sim\; \mathcal{N}(m_0,\, v_0),
\]
with default hyperparameters $(m_0, v_0) = (0, 1)$ and known observation 
standard deviation $\sigma = 1$. The reward model is taken to be the conjugate 
normal--normal posterior mean,
\[
\widehat{q}(a) \;=\; \frac{m_0/v_0 \;+\; N_a\,\widehat{\mu}_a/\sigma^2}
{1/v_0 \;+\; N_a/\sigma^2}.
\]

\paragraph{Confidence Interval via Bootstrap.}
To account for the sampling variability introduced by estimating $\widehat{q}$ 
from data, the confidence interval is constructed via a nonparametric bootstrap. 
For each replication $b = 1, \ldots, B$, a bootstrap sample of size 
$T_{\mathrm{off}}$ is drawn with replacement from the offline dataset. The 
arm-wise reward model $\widehat{q}^{(b)}$ is refit on this bootstrap sample, 
and the DR estimate $\widehat{\Psi}_{\mathrm{DR}}^{(b)}$ is recomputed using 
the resampled actions, rewards, target-policy probabilities, and importance 
weights. 
The resulting $(1-\alpha)$ bootstrap confidence interval is
\[
\left[\,\widehat{\Psi}_{\mathrm{DR}} \;\pm\; z_{1-\alpha/2}\,
\widehat{\sigma}_{\mathrm{boot}}\,\right], \quad \text{where} \quad
\widehat{\sigma}^2_{\mathrm{boot}} \triangleq \frac{1}{B}\sum^{B}_{b=1}  (\widehat{\Psi}_{\mathrm{DR}}^{(b)} -\bar{\Psi}_{\mathrm{DR}} )^2.
\]

\subsubsection{Contextual Adaptive Doubly Robust (CADR)}

Motivated by Contextual Adaptive Doubly Robust (CADR)~\citep{bibaut2021post}, we use a variance-stabilized IPW method as a baseline for off-policy evaluation. This approach modifies the IPW with an adaptive variance-stabilization component. Let
\begin{align*}
    \Gamma_{T_{\text{offline}}} \triangleq \left(\frac{1}{T_{\text{offline}}}\sum_{t=1}^{T_{\text{offline}}} \frac{1}{\widehat{\sigma}_t^{}}  \right)^{-1}  \quad \text{and,} \quad \widehat{\Psi}_{T_{\text{offline}}} = \Gamma_{T_{\text{offline}}} \left( \frac{1}{T_{\text{offline}}}\sum_{t=1}^{T_{\text{offline}}}  \frac{1}{\widehat{\sigma}_t^{}} W_t R_t \right)
\end{align*}
where,
\begin{align*}
   \widehat{\sigma}_t^2
\triangleq
\left(
\frac{1}{t-1}\sum_{s=1}^{t-1} ( W_s R_s)^2
\right)
-
\left(
\frac{1}{t-1}\sum_{s=1}^{t-1}  W_s R_s
\right)^2 
\end{align*}

Thus, for nominal confidence level \(1-\alpha\), the CADR confidence interval takes the form
\[
\left[
 \widehat{\Psi}_{T_{\text{offline}}} \pm z_{1-\alpha/2}\,\frac{ \Gamma_{T_{\text{offline}}}}{\sqrt{\Toff}}
\right].
\]

\subsubsection{Empirical Likelihood (ELFCB)}

ELFCB~\citep{karampatziakis2020empirical} constructs confidence bounds for the target policy value using an
empirical-likelihood formulation built from the weighted logged rewards. We implement this method using the code available in the following public repository: \href{https://github.com/pmineiro/elfcb}{https://github.com/pmineiro/elfcbre}

\end{document}